\documentclass{article}
\usepackage[a4paper,margin=1.2in]{geometry}
\usepackage{times}
\usepackage{helvet}
\usepackage{courier}
\usepackage[hyphens]{url}
\usepackage{graphicx}
\usepackage{natbib}
\usepackage{caption}
\usepackage{amsmath}
\usepackage{amssymb}
\usepackage{bm}
\usepackage{mathrsfs}
\usepackage{url}
\usepackage{hyperref}
\usepackage{wrapfig}
\usepackage[boxruled,linesnumbered]{algorithm2e}
\usepackage{multirow}
\usepackage{algpseudocode}
\usepackage{booktabs}
\usepackage{xcolor}
\usepackage[most]{tcolorbox}
\usepackage{subcaption}
\usepackage[capitalise]{cleveref}
\input{definitions.tex}

\begin{document}

\title{Local Gradient Regulation Stabilizes Federated Learning under Client Heterogeneity}

\author{
Ping Luo$^{1}$\quad
Jiahuan Wang$^{1}$\quad
Ziqing Wen$^{1}$\quad
Tao Sun$^{1}$\thanks{Corresponding authors. Contributed equally.}\quad
Dongsheng Li$^{1}$\footnotemark[1]
\\[6pt]
\normalsize $^{1}$National Key Laboratory of Parallel and Distributed Computing,\\
\normalsize College of Computer Science and Technology,\\
\normalsize National University of Defense Technology, ChangSha, 410073, China
\\[6pt]
\normalsize
\texttt{luoping@nudt.edu.cn, wangjiahuan@nudt.edu.cn, zqwen@nudt.edu.cn,}\\
\normalsize
\texttt{suntao.saltfish@outlook.com, dsli@nudt.edu.cn}
}

\date{}
\maketitle

\begin{abstract}
Federated learning (FL) enables collaborative model training across distributed clients without sharing raw data, yet its stability is fundamentally challenged by statistical heterogeneity in realistic deployments. Here, we show that client heterogeneity destabilizes FL primarily by distorting local gradient dynamics during client-side optimization, causing systematic drift that accumulates across communication rounds and impedes global convergence. This observation highlights local gradients as a key regulatory lever for stabilizing heterogeneous FL systems. Building on this insight, we develop a general client-side perspective that regulates local gradient contributions without incurring additional communication overhead. Inspired by swarm intelligence, we instantiate this perspective through Exploratory--Convergent Gradient Re-aggregation (ECGR), which balances well-aligned and misaligned gradient components to preserve informative updates while suppressing destabilizing effects. Theoretical analysis and extensive experiments, including evaluations on the LC25000 medical imaging dataset, demonstrate that regulating local gradient dynamics consistently stabilizes federated learning across state-of-the-art methods under heterogeneous data distributions.

\end{abstract}

\section{Introduction}
\label{sec:intro}
Federated Learning (FL)~\citep{mcmahan2017} has emerged as a distributed machine learning paradigm that enables collaborative model training without requiring clients to share their raw data. As data silos and increasingly stringent privacy regulations continue to constrain centralized learning, FL offers an effective solution by keeping sensitive data localized while only exchanging model updates. In recent years, FL has achieved remarkable success across a wide range of domains, including computer vision, natural language processing, and recommender systems~\citep{kairouz2021advances}. In particular, its privacy-preserving nature makes FL highly attractive for medical and healthcare applications, such as cross-institutional medical image analysis~\citep{lee2024international}, electronic health record modeling~\citep{sadilek2021privacy}, and disease risk prediction~\citep{dayan2021federated}, where data sharing is often severely restricted. These advances highlight the practical potential of FL as a foundation for large-scale, privacy-aware intelligent systems.

Despite this promise, FL faces fundamental challenges in realistic settings, most notably the prevalence of statistical heterogeneity across clients~\citep{ma2022state}. In practice, client data are rarely independent and identically distributed (IID), violating a key assumption underlying classical federated optimization algorithms such as FedAvg~\citep{mcmahan2017}. non-IID data distributions can significantly slow convergence, induce training instability, and lead to substantial degradation in model performance~\citep{wang2020tackling}. During local training, heterogeneous data generate client-specific update directions that may deviate markedly from the global optimum~\citep{zhang2021adaptive}. The accumulation of such gradient discrepancies constitutes a primary source of optimization difficulty and ultimately limits the effectiveness and scalability of FL in real-world deployments.

\begin{figure}[ht]
\centering
  \includegraphics[width=0.8\columnwidth]{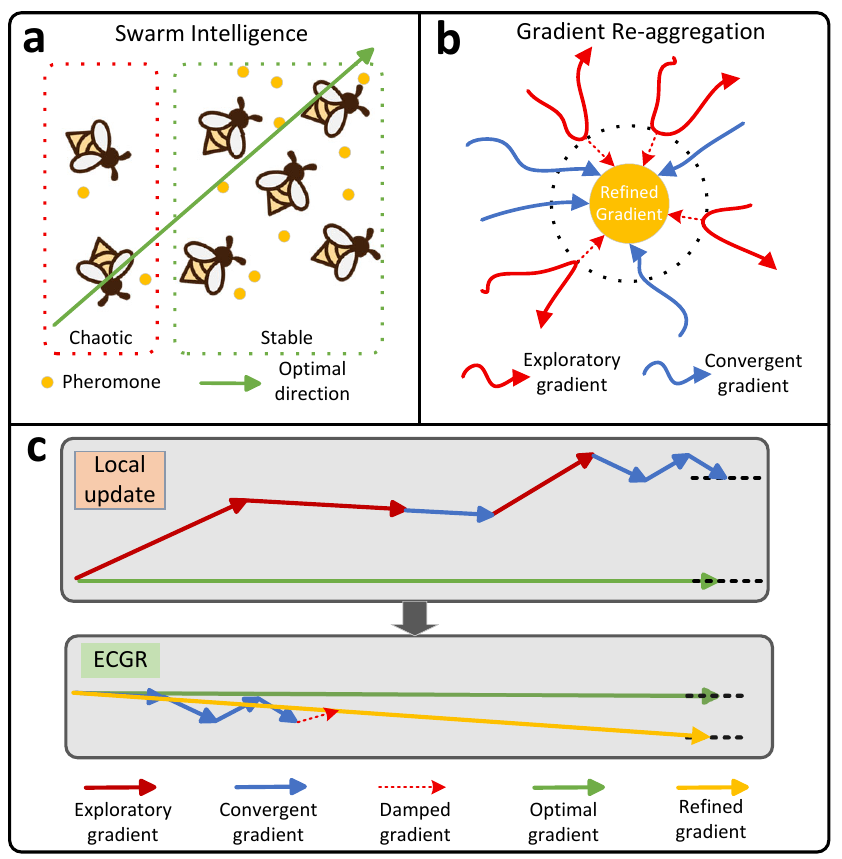}
  \caption{Framework of the \textbf{ECGR} strategy. 
  \textbf{(a)} Illustration of swarm intelligence in honeybees: foraging paths typically consist of both chaotic and stable directions, with the stable direction dominating collective behavior. 
  \textbf{(b)} Inspired by swarm intelligence, local gradients in FL are categorized into \textit{exploratory gradients} and \textit{convergent gradients}, which are re-aggregated such that convergent gradients dominate the resulting update. 
  \textbf{(c)} A two-dimensional visualization of aggregated gradients, illustrating how ECGR reduces gradient deviation induced by data heterogeneity.}
  \label{fig:ECGR}
\end{figure}

In our previous work~\citep{luo2025bherd}, we systematically investigated the optimization behavior of FL under non-IID data and identified a critical mechanism underlying performance degradation. Specifically, statistical heterogeneity manifests itself most directly through its influence on local gradients: non-IID data reshape both the direction and magnitude of client-side updates during local training, thereby inducing pronounced \emph{client drift}. Rather than arising solely from global aggregation, this gradient-level distortion emerges locally and accumulates across training rounds, motivating a re-examination of how local updates are formed before communication.

Guided by this insight, the central idea of the present work is to mitigate client drift by operating directly on local gradients at the client side, without introducing any additional communication overhead or modifying the existing FL protocol. Inspired by swarm intelligence observed in honeybee foraging behavior~\citep{tereshko2005collective,11008686}, we propose a novel gradient re-aggregation strategy termed ECGR (Exploratory--Convergent Gradient Re-aggregation). As illustrated in \cref{fig:ECGR}~(a), although a subset of bees acts as explorers and follows paths that deviate from the optimal route, collective behavior is ultimately governed by bees carrying stable pheromone signals and consistent directional information. Mapping this mechanism to a single client in FL (\cref{fig:ECGR}~(b)), we regard each local gradient as an individual bee, while the information encoded in the gradient corresponds to pheromone signals. Local gradients that deviate substantially from the optimal update direction are identified as \emph{exploratory gradients}; although noisy, they still contain information essential for model convergence. In contrast, gradients that are well aligned with the optimal direction are defined as \emph{convergent gradients} and serve as the dominant contributors to the update. As shown in \cref{fig:ECGR}~(c), ECGR preserves the full contribution of convergent gradients while extracting useful information from exploratory gradients through a damped refinement mechanism. The resulting gradient is rescaled to match the norm of the original local update, yielding a more stable and robust optimization trajectory.

Together, these results establish a general client-side optimization perspective for FL that distills local gradients to mitigate client drift, and demonstrate both theoretically and empirically the effectiveness of ECGR under heterogeneous data distributions.

\section{Related Work}
\label{sec:related}
\noindent \textbf{Federated optimization under non-IID data.}
A substantial body of recent work has focused on mitigating the adverse effects of statistical heterogeneity in federated learning. Early efforts primarily addressed non-IID data by modifying aggregation rules or introducing control variates to correct biased local updates. Representative approaches include FedProx~\citep{li2020federated1}, which constrains local updates through a proximal term, and SCAFFOLD~\citep{karimireddy2020scaffold}, which employs control variates to reduce client drift. Subsequent studies explored adaptive aggregation and normalization strategies, such as FedNova~\citep{wang2020tackling} and FedAvgM~\citep{hsu2019measuring}, to stabilize convergence under heterogeneous data distributions. More recently, personalized and clustered federated learning methods have been proposed to explicitly account for client heterogeneity by learning multiple client-specific or group-level models~\citep{fallah2020personalized,ghosh2022efficient}. While these approaches have demonstrated effectiveness, they typically operate at the level of global aggregation or client participation, leaving the structure of local optimization dynamics largely unexamined.

\noindent \textbf{Leveraging local gradients in federated learning.}
Beyond aggregation-centric strategies, an emerging line of work has investigated how local gradient information can be exploited to improve federated training. Several studies use gradient statistics to guide client selection or weighting, prioritizing clients whose updates are more informative or reliable~\citep{nishio2019client,tang2022fedcor,li2022pyramidfl}. Other works leverage local gradients to identify high-quality or representative data subsets, thereby reducing the impact of noisy or biased local samples~\citep{schutte2024fedgs,9488723}. In addition, gradient-based screening mechanisms have been explored to detect stragglers or anomalous updates in heterogeneous environments~\citep{pillutla2022robust}. These methods demonstrate that local gradients encode rich information about data quality and optimization behavior. However, most existing approaches utilize gradients indirectly---for client or data selection---rather than directly operating on the local gradient set itself. In contrast, a smaller number of recent studies have begun to consider explicit gradient-level manipulation, such as gradient clipping~\citep{zhou2025poster}, filtering~\citep{han2024degafl}, or reweighting~\citep{li2023revisiting}, to improve robustness. Our work aligns with this emerging direction but differs fundamentally in that it performs structured distillation of local gradients at the client side, extracting useful information from noisy updates without discarding them or increasing communication overhead.

\noindent \textbf{Federated learning in computational pathology.}
Computational pathology has emerged as a prominent application domain for federated learning, driven by the sensitivity, scale, and institutional fragmentation of medical imaging data~\citep{adnan2022federated}. Recent studies have demonstrated the feasibility of FL for whole-slide image analysis~\citep{li2021dual}, tumor classification~\citep{al2024brain}, and prognosis prediction~\citep{feng2024robustly,tahir2025federated} across distributed pathology centers. non-IID data are particularly pronounced in this setting due to variations in staining protocols, scanners, patient demographics, and clinical practices~\citep{xiang2023automatic,lu2021data}. To address these challenges, prior work has explored domain adaptation, normalization, and personalized FL strategies tailored to pathology data~\citep{antunes2022federated,lu2022federated}. Nevertheless, optimization instability induced by heterogeneous local gradients remains a critical bottleneck. By directly distilling local gradients before aggregation, the proposed ECGR strategy offers a complementary optimization perspective that is well suited to the intrinsic heterogeneity of computational pathology and other privacy-sensitive medical applications.

\section{Method}
\label{sec:method}
We begin by introducing the relevant definitions and notations for the Federated Averaging (FedAvg)~\citep{mcmahan2017} training process, including both local and global update stages. Building upon these foundations, we propose a new mechanism, termed \emph{ECGR}, which refines the local update strategy. We then formalize its overall workflow and present the corresponding algorithmic design in detail.
  
\subsection{Preliminaries: Federated Averaging (FedAvg)}
\subsubsection{Clients and Datasets.}  
Consider $N$ clients, each associated with a local dataset $\mathcal{D}_{i} \subset \mathcal{D}$ $(i=1,2,\dots,N)$, where $\boldsymbol{x}_{i}\in\mathcal{D}_{i}$ denotes a training sample. The client sampling weights follow the conventional setting in FedAvg, i.e.,  
\begin{equation}
p_i = \frac{|\mathcal{D}_i|}{|\mathcal{D}|}, \quad \text{with } |\mathcal{D}|=\sum_{i=1}^N |\mathcal{D}_i|.
\end{equation}

\subsubsection{Communication Rounds and Local Updates.}  
The FL process proceeds for $T\geq1$ communication rounds, where the server maintains global parameters $\boldsymbol{w}_{t}$ for round $t=0,1,\dots,T$. At each round $t$, client $i$ trains for $E$ local epochs, which correspond to  $\tau_i = E\frac{|\mathcal{D}_i|}{B}$ ($ \tau_{i} \geq 1 $) local SGD iterations with batch size $B$. Let $\boldsymbol{w}_{(t,i)}^{\lambda}$ denote the local model parameters at iteration $\lambda=0,1,\dots,\tau_i$. The local update rule is  
\begin{equation}
\label{eq:local-update}
\boldsymbol{w}_{(t,i)}^{\lambda+1}=\boldsymbol{w}_{(t,i)}^\lambda - \eta_l \nabla F_i(\boldsymbol{w}_{(t,i)}^\lambda;\boldsymbol{x}_{s_i}),
\end{equation}
where $\eta_l$ is the local learning rate, $F_i(\cdot)$ is the local loss function, and $s_i=\{1,2,\dots,\tau_i\}$ denotes the permutation of mini-batches.  

\subsubsection{Local and Global Aggregation Gradients.}  
For each client $i$, the gradients computed on individual mini-batches are first collected and then aggregated to obtain
\begin{equation}
\label{eq:local_grad_set}
\boldsymbol{g}_{(t,s_i)} := \sum 
\underbrace{\left\{ \eta_l \nabla F_i(\boldsymbol{w}_{(t,i)}^\lambda;\boldsymbol{x}_{s_i}) \right\}_{\lambda=1}^{\tau_i}}_{\text{local gradient set}}.
\end{equation}
After locally aggregating the gradients within each local gradient set (typically by averaging), the locally updated training gradient \( \boldsymbol{g}_{(t, s_i)} \) is obtained. It is then transmitted to the parameter server for global aggregation (typically by weighted averaging) as follows:
\begin{equation}
    \boldsymbol{G}_{t} = \sum_{i=1}^N p_i \boldsymbol{g}_{(t,s_i)}
\end{equation}
It should be noted that all the preceding operations are performed on the individual clients, whereas this step and the subsequent global update are carried out on the parameter server.

\subsubsection{Global Update.}
After obtaining the global gradient $\boldsymbol{G}_{t}$ at round $t$, the global model is updated via a straightforward SGD step:
\begin{equation}
\label{eq:global-aggregation}
\boldsymbol{w}_{t+1} = \boldsymbol{w}_{t} - \eta_g \boldsymbol{G}_{t}.
\end{equation}
where $\eta_g$ is the global learning rate. In this paper, we set the local learning rate $\eta_l$ to be the same across all clients, and fix the global learning rate $\eta_g=1$.

\subsection{Exploratory-Convergent Gradient Re‐aggregation (ECGR)}
In our proposed ECGR method, the local gradient set at each client is selectively sampled (with replacement) before local aggregation. The selection strategy consists of three steps:  
\begin{enumerate}
    \item \textbf{Magnitude Ranking}: Select the top half of the local gradients based on their magnitudes, ensuring that the resulting aggregated vector attains the smallest $\ell_2$ discrepancy. These selected gradients are denoted as the \textit{Convergent Gradients}. 
    \item \textbf{Attenuated Extraction}: The remaining gradients after the Magnitude Ranking operation are denoted as the \textit{Exploratory Gradients}. Each of them is multiplied by a damping factor $\beta\in [0,1]$ (typically $\beta=0.1 \sim 0.5$ in experiments) to reduce their influence.  
    \item \textbf{Re-aggregation}: The Convergent and Exploratory Gradients obtained from the previous steps are locally aggregated, and the resulting vector is rescaled to match the $\ell_2$ norm of the original aggregated gradient.  
\end{enumerate}
The formal definitions of each step are provided below.

\subsubsection{Magnitude Ranking} 
We adopt the herding-based greedy strategy from~\cite{lu2022grab} to sequentially sample one half of the gradients from the local gradient set 
$\left\{ \eta_l \nabla F_i(\boldsymbol{w}_{(t,i)}^\lambda;\boldsymbol{x}_{s_i}) \right\}$. 
The selected gradients form a subset 
$\pi_i = \{e_1, e_2, \dots, e_k\}$, 
where $e_\lambda$ denotes the index induced by the permutation $s_i$ after sorting, and 
$k = \left\lfloor \tau_i / 2 \right\rfloor$.

Firstly, let $\boldsymbol{S}_{0} = \boldsymbol{0}$ and $R_{0} = s_i$. At the $\lambda$-th ($\lambda \in [1,k]$) step, we select
\begin{equation}
\label{eq:selection_rule}
    e_{\lambda} = \arg\min_{e_{\lambda} \in R_{\lambda-1}} \, \big\| \boldsymbol{S}_{\lambda-1} + \eta_l \nabla F_i(\boldsymbol{w}_{(t,i)}^{e_{\lambda}};\boldsymbol{x}_{s_i}) \big\|,
\end{equation}
And update
\begin{equation}
    \boldsymbol{S}_{\lambda} := \boldsymbol{S}_{\lambda-1} + \eta_l \nabla F_i(\boldsymbol{w}_{(t,i)}^{e_{\lambda}};\boldsymbol{x}_{s_i}), 
    \qquad 
    R_{\lambda} := R_{\lambda-1} \setminus \{e_{\lambda}\}.
\end{equation}
Finally, we obtain
\begin{equation}
    \pi_i = \arg\min_{\pi_i \subset s_i} \big\| \boldsymbol{g}_{(t,\pi_i)} \big\|,
    \qquad
    \boldsymbol{g}_{(t,\pi_i)} = \sum_{\lambda=1}^{\left\lfloor \tau_i / 2 \right\rfloor} \eta_l \nabla F_i(\boldsymbol{w}_{(t,i)}^\lambda;\boldsymbol{x}_{\pi_i}) .
\end{equation}

In this step, the selected set $\pi_i$ represents the ``convergent'' portion of the client's gradient set, as it contains gradients that are directionally consistent with the global descent trend while filtering out those dominated by local noise or outliers. 
This selection helps stabilize the optimization process, leading to faster global convergence and better generalization.
However, the experimental findings in~\cite{luo2025bherd} suggest that applying this step alone may lead to the loss of beneficial gradient information. Therefore, Attenuated Extraction is further required to extract additional useful gradients.

\subsubsection{Attenuated Extraction}
After obtaining the gradient index set $\pi_i$ through the Magnitude Ranking step, the remaining gradient set can be directly derived as:
\begin{equation}
    \pi_i^{\prime} = s_i\setminus \pi_i,
    \qquad
    \boldsymbol{g}_{(t,\pi_i^{\prime})} = \sum_{\lambda=1}^{\left\lfloor \tau_i / 2 \right\rceil
} \eta_l \nabla F_i(\boldsymbol{w}_{(t,i)}^\lambda;\boldsymbol{x}_{\pi_i^{\prime}}) .
\end{equation}

In contrast, the set $\pi_i^{\prime}$ corresponds to the ``exploratory'' gradients, which include components that may still contribute positively to global convergence but also contain a higher level of stochastic or biased information. 
To balance exploration and stability, these gradients are scaled by a damping factor $\beta \in [0,1]$, which mitigates the influence of potentially harmful updates while retaining the beneficial exploratory directions that enhance model robustness and prevent premature convergence.

\subsubsection{Re-aggregation}
After obtaining the gradient subsets from both Magnitude Ranking $\pi_i$ and Alignment Ranking $\pi_i^{\prime}$, the next step is to combine them to form the re-aggregated gradient. 
\begin{equation}
\label{eq:re_agg}
    \boldsymbol{g}_{(t,s_i)}^{\prime}=\gamma_i(\boldsymbol{g}_{(t,\pi_i)} + \beta \boldsymbol{g}_{(t,\pi_i^{\prime})}),
    \qquad
    \gamma_i = \| \boldsymbol{g}_{(t,s_i)} \|/ \|\boldsymbol{g}_{(t,\pi_i)} + \beta \boldsymbol{g}_{(t,\pi_i^{\prime})} \|
\end{equation}

As shown in \cref{eq:re_agg}, the re-aggregation process balances the ``convergent'' and ``exploratory'' components through the damping factor $\beta \in [0,1]$, producing the refined local update $\boldsymbol{g}_{(t,\pi_i)} + \beta \boldsymbol{g}_{(t,\pi_i^{\prime})}$. 
Here, the scaling coefficient $\gamma_i$ is introduced to ensure that the re-aggregated gradient preserves the same descent magnitude as the original gradient $\boldsymbol{g}_{(t,s_i)}$, while allowing a directional adjustment. 
This design implies that our ECGR method modifies only the aggregation direction of local gradients, rather than their overall update strength, thereby maintaining optimization stability and consistency across clients.

Finally, we plug the re-aggregated local gradients from each client into the standard FedAvg procedure to obtain the global aggregated gradient $\boldsymbol{G}_{t}^{\prime}$ and perform the global model update
$\boldsymbol{G}_{t}^{\prime}$ and perform the global update:
\begin{equation}
   \boldsymbol{G}_{t}^{\prime} = \sum_{i=1}^N p_i \boldsymbol{g}_{(t,s_i)}^{\prime},
   \qquad
   \boldsymbol{w}_{t+1} = \boldsymbol{w}_{t} - \boldsymbol{G}_{t}^{\prime}.
\end{equation}

This step ensures that the proposed ECGR mechanism remains fully compatible with the conventional federated optimization pipeline, introducing no additional synchronization or communication overhead. 
By incorporating directionally refined local updates, the global model is guided toward a more stable and consistent descent trajectory, effectively mitigating the adverse effects of heterogeneous or noisy local data while accelerating convergence across rounds.

\subsection{Algorithm Description}

\SetKwInput{KwData}{Require}
\begin{algorithm}[h]
\setlength{\abovedisplayskip}{2pt}
\setlength{\belowdisplayskip}{2pt}
\setlength{\abovedisplayshortskip}{0pt}
\setlength{\belowdisplayshortskip}{0pt}
\caption{FedAvg-ECGR}
\label{alg:FedAvg-ECGR}
\KwData{Total global round $T$, local dataset $\mathcal{D}_{i}$ ($\boldsymbol{x}_{i}\in\mathcal{D}_{i}$), local iterations $\tau_i$, initialized weight $\boldsymbol{w}_{0}$, initialized order $ s_i $ at client $i$, learning rate $\eta > 0$}
  \For{each round $t=0,\dots,T-1$}
  {
    Parameter server send the global model $\boldsymbol{w}_{t}$ to all participating clients\;
    \For{each client $i=1,...,N$}
    {
      \For{each local iteration $\lambda=0,1,\dots,\tau_i$}
      {
        Initialize the local model $\boldsymbol{w}_{(t,i)}^{\lambda} \leftarrow \boldsymbol{w}_{t}$ \;
        Local update $ \boldsymbol{w}_{(t,i)}^{\lambda+1}=\boldsymbol{w}_{(t,i)}^\lambda - \eta \nabla F_i(\boldsymbol{w}_{(t,i)}^\lambda;\boldsymbol{x}_{s_i}) $ \;
      }
      Store the local gradient set $ \left\{ \eta \nabla F_i(\boldsymbol{w}_{(t,i)}^\lambda;\boldsymbol{x}_{s_i}) \right\}_{\lambda=1}^{\tau_i} $ \;
      \textbf{ECGR}:
          \[
          \begin{split}
          &\pi_i \leftarrow \arg\min_{\pi_i \subset s_i} \big\| \boldsymbol{g}_{(t,\pi_i)} \big\|
          \quad\#\ \text{Magnitude Ranking}
          \\
          &\pi_i^{\prime} = s_i\setminus \pi_i,  \quad \beta \boldsymbol{g}_{(t,\pi_i^{\prime})}
          \quad\#\ \text{Attenuated Extraction}
          \\
          &\boldsymbol{g}_{(t,s_i)}^{\prime}=\gamma_i(\boldsymbol{g}_{(t,\pi_i)} + \beta \boldsymbol{g}_{(t,\pi_i^{\prime})})
          \quad\#\ \text{Re-aggregation}
          \end{split}
          \]
      
    }
    Parameter server receive $\boldsymbol{g}_{(t,s_i)}^{\prime}$ from all clients\;
    Global aggregation $\boldsymbol{G}_{t}^{\prime} = \sum_{i=1}^N p_i \boldsymbol{g}_{(t,s_i)}^{\prime}$\;
    Global update $\boldsymbol{w}_{t+1} = \boldsymbol{w}_{t} - \boldsymbol{G}_{t}^{\prime}$\;
  }
  \Return $\boldsymbol{w}_{T}$\;  
\end{algorithm}

To verify the effectiveness of the proposed \textbf{ECGR} strategy, we incorporate it into the classical \textbf{FedAvg} algorithm, resulting in the \textbf{FedAvg-ECGR} algorithm shown in \cref{alg:FedAvg-ECGR}. This integration enables us to evaluate how ECGR enhances the optimization behavior within the standard FL framework, where a central server coordinates multiple distributed clients to collaboratively train a shared model without exchanging raw data.

In each global communication round, the parameter server transmits the current global model $\boldsymbol{w}_t$ to all participating clients. Each client $i$ performs $\tau_i$ local updates on its private dataset $\mathcal{D}_i$, generating a sequence of local gradients $\{\eta \nabla F_i(\boldsymbol{w}_{(t,i)}^\lambda; \boldsymbol{x}_{s_i})\}_{\lambda=1}^{\tau_i}$ (Lines~1--8 in \cref{alg:FedAvg-ECGR}). 

After completing local training, the client performs the \textbf{ECGR} procedure, which refines the local gradients through a three-step process: 
(1) \textit{magnitude ranking} selects gradients with smaller norms to form subset $\pi_i$, 
(2) \textit{attenuated extraction} scales the complementary subset $\pi_i^{\prime}$ by an attenuation factor $\beta$, 
and (3) \textit{re-aggregation} combines both subsets to yield the adjusted local gradient $\boldsymbol{g}_{(t,s_i)}^{\prime}$ (Lines~9--10 in \cref{alg:FedAvg-ECGR}). 
Each client then transmits $\boldsymbol{g}_{(t,s_i)}^{\prime}$ to the server, which performs weighted aggregation to obtain $\boldsymbol{G}_{t}^{\prime}$ and updates the global model $\boldsymbol{w}_{t+1}$ accordingly (Lines~11--14 in \cref{alg:FedAvg-ECGR}). 
This process repeats until convergence, producing the final global model $\boldsymbol{w}_{T}$.

As illustrated, the baseline FedAvg framework corresponds to Lines~1--7 and Lines~11--15 in \cref{alg:FedAvg-ECGR}. The proposed \textbf{ECGR} mechanism extends this framework by introducing an additional local operation at each client (Lines~8--10), which serves as an effective yet lightweight gradient refinement step.

This modification offers two primary advantages:
\begin{itemize}
    \item \textbf{Communication efficiency.} Compared with FedAvg, \emph{ECGR} incurs no additional communication cost, since each client still uploads only a single aggregated gradient $\boldsymbol{g}_{(t,s_i)}^{\prime}$ to the server. This property is particularly desirable for bandwidth-limited federated environments.
    \item \textbf{Structural compatibility.} \emph{ECGR} maintains the original structure of FedAvg, including both local and global update procedures, ensuring seamless compatibility with existing FL systems based on the FedAvg framework.
\end{itemize}

However, \emph{ECGR} introduces two additional costs.  
First, there is a storage overhead: as shown in Line~8 of \cref{alg:FedAvg-ECGR}, the storage requirement of \emph{ECGR} is approximately $\tau_i$ times that of FedAvg, since all local gradients must be retained for selection rather than discarded after each update.  
Second, there is a computational overhead: the gradient selection in Line~9 has a complexity of $O(\tau_i!)$, slightly higher than that of FedAvg.  

Nevertheless, \emph{ECGR} aligns with the core design principle of federated learning—trading inexpensive local computation and memory for reduced communication cost between clients and the central server. Extensions of ECGR to other state-of-the-art federated learning algorithms, are presented in Appendix \cref{sec:extend-algo}.

\section{Empirical Evaluation}
\label{sec:exp}

In this section, we comprehensively evaluate the effectiveness of the proposed \emph{MAGS} algorithm on several widely used image classification benchmarks and a real-world medical image diagnosis task. We further analyze its performance in comparison with classical and state-of-the-art FL baselines to demonstrate its advantages in both accuracy and stability.
The complete implementation and experimental setup are publicly available at \href{https://github.com/NUDTPingLuo/ECGR}{\tt{https://github.com/NUDTPingLuo/ECGR}} to facilitate reproducibility and future research.

\subsection{Benchmark Image Classification}
\label{sec:benchmark_image_exp}

\begin{figure*}[htbp]
\centering
  \includegraphics[width=1.0\textwidth]{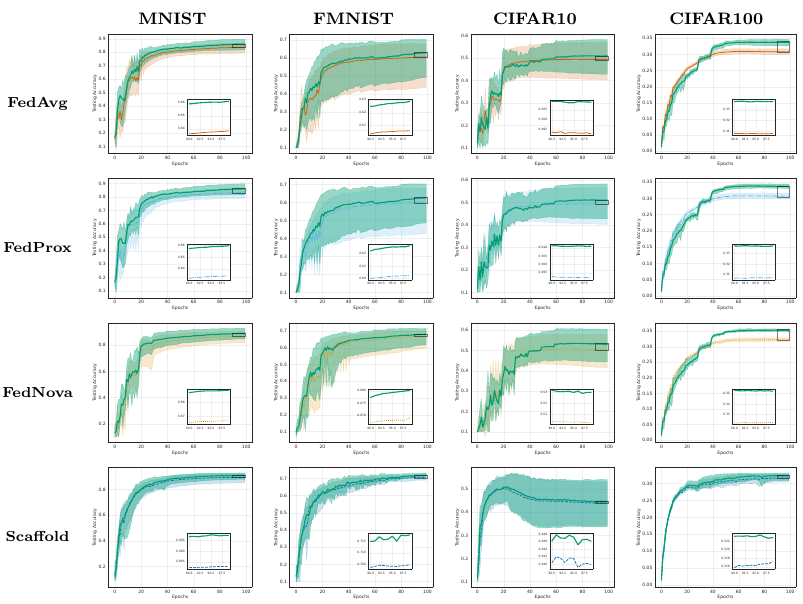}
\caption{Global model testing accuracy curves for different FL algorithms across multiple datasets. Each row corresponds to one algorithm, and each column presents the results on a particular dataset. The green solid line indicates the accuracy trajectory of the ECGR-extended variant, whereas the remaining curves represent the corresponding standard baselines. Each plot shows the mean testing accuracy along with the upper and lower bounds, computed from runs using 5 different random seeds.}
\label{fig:accuracy_curves}
\end{figure*}

\noindent\textbf{Datasets \& model \& settings}.
We conducted experiments on MNIST~\citep{lecun1998gradient}, Fashion-MNIST (FMNIST)~\citep{xiao2017fashion}, CIFAR-10, and CIFAR-100~\citep{krizhevsky2009learning}. 
MNIST and FMNIST contain 60k grayscale training images of size $28\times28$, while CIFAR-10 and CIFAR-100 contain 50k RGB training images of size $32\times32$. 
Each dataset has 10k test images used to evaluate model performance. 

\noindent\textbf{Model}.
For MNIST and FMNIST, we adopt the classical LeNet architecture~\citep{lecun1998gradient}, which consists of two convolutional layers followed by three fully connected layers. 
The forward propagation involves ReLU activations after each convolution and fully connected layer, with max pooling applied after the convolutional layers. 
For CIFAR-10 and CIFAR-100, we use a deeper convolutional neural network (CNN) tailored for $32\times32$ RGB images. 
The network comprises three convolutional blocks with increasing channel dimensions (32$\rightarrow$64$\rightarrow$128$\rightarrow$256), each block containing convolution, batch normalization, ReLU activation, and max pooling layers. 
The final feature maps are flattened and fed into a fully connected layer that outputs class predictions, with a log-softmax function applied at the output for numerical stability.

\noindent\textbf{Settings}.
All experiments were performed on a workstation equipped with an NVIDIA RTX~4070~Ti GPU, simulating 10 federated clients.
Each client trains on its local dataset for one epoch per global round. 
The local training data on each client is sampled from the total training set according to a Dirichlet distribution to simulate extreme non-IID scenarios, with the concentration parameter $\alpha$ set to $0.01$. For each dataset, five distinct random seeds (0, 1, 42, 999, and 2025) are used to generate different Dirichlet partitions, ensuring statistical reliability of the reported results.
A minimum data size of two batches is enforced on each client to ensure that at least two local gradients can be computed for selection.  
All models are trained using stochastic gradient descent (SGD) with a momentum of $0.9$, and the total number of global training rounds is set to $T=100$.
The training hyperparameters are configured consistently across all datasets: the learning rate is initialized at $0.001$ and decays by a factor of two every $10$ rounds, while the batch size is fixed at $128$ for all experiments.

\noindent\textbf{Baselines}.
We compare our method with several representative federated learning baselines, including FedAvg~\citep{mcmahan2017}, FedProx~\citep{li2020federated1}, FedNova~\citep{wang2020tackling} and Scaffold~\citep{karimireddy2020scaffold}. 
FedAvg performs simple model averaging across clients. 
FedProx extends FedAvg by adding a proximal term to the local objective, stabilizing training under statistical heterogeneity. 
FedNova further normalizes local updates to eliminate objective inconsistency caused by varying local epoch numbers. 
Scaffold mitigates client drift in non-IID scenarios by introducing control variates to correct local updates. 
All methods are trained under the same experimental setup for fair comparison, and the corresponding extended variants are provided in \cref{sec:extend-algo}.

\noindent\textbf{Results \& discussions}. 
We adopt a damping factor of $\beta = 0.2$ in the ECGR strategy, and \cref{fig:accuracy_curves} shows the performance on the test datasets.
There are several noteworthy observations: (i)~For the final global model accuracy, ECGR consistently improves performance across all selected datasets and baseline methods, yielding absolute accuracy gains of approximately 1\%--2\% over their corresponding baselines and thereby demonstrating its overall effectiveness. (ii)~ECGR exhibits convergence trajectories that closely align with those of their corresponding baselines, indicating that ECGR refines the optimization process along the original trajectories of each method, in agreement with the procedural logic presented in~\cref{alg:FedAvg-ECGR}.

\subsection{Ablation}
\label{sec:ablation}

\begin{figure}[htbp]
    \centering
    \begin{subfigure}{0.32\textwidth}
        \centering
        \includegraphics[width=\textwidth]{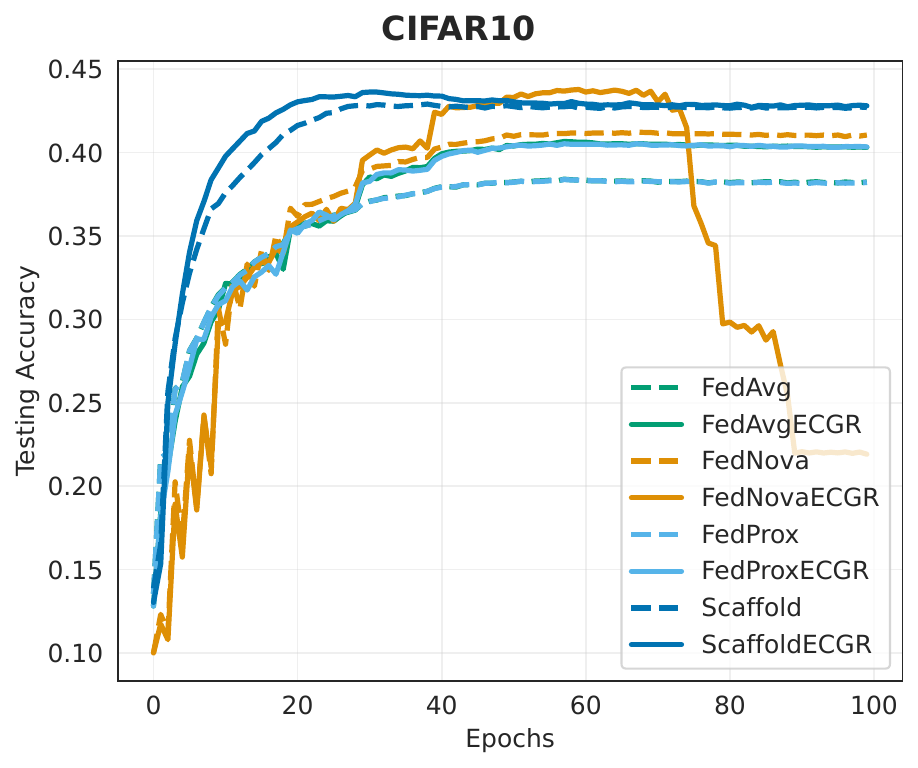}
        \caption{$\eta=0.0001$}
    \end{subfigure}
    \hfill
    \begin{subfigure}{0.32\textwidth}
        \centering
        \includegraphics[width=\textwidth]{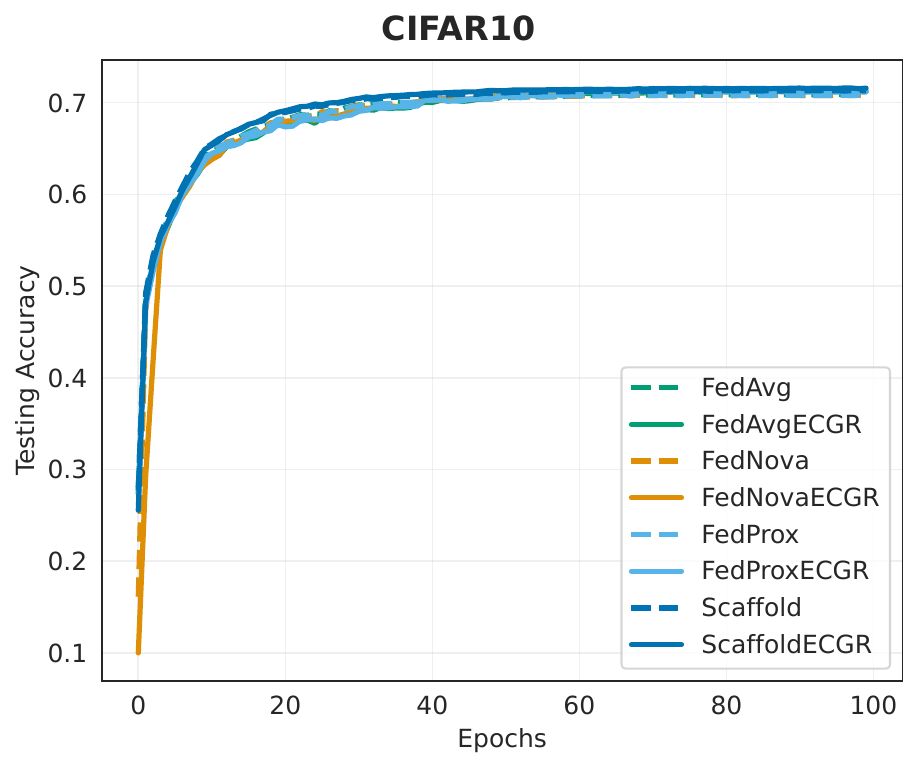}
        \caption{$\alpha=1$}
    \end{subfigure}
    \hfill
    \begin{subfigure}{0.32\textwidth}
        \centering
        \includegraphics[width=\textwidth]{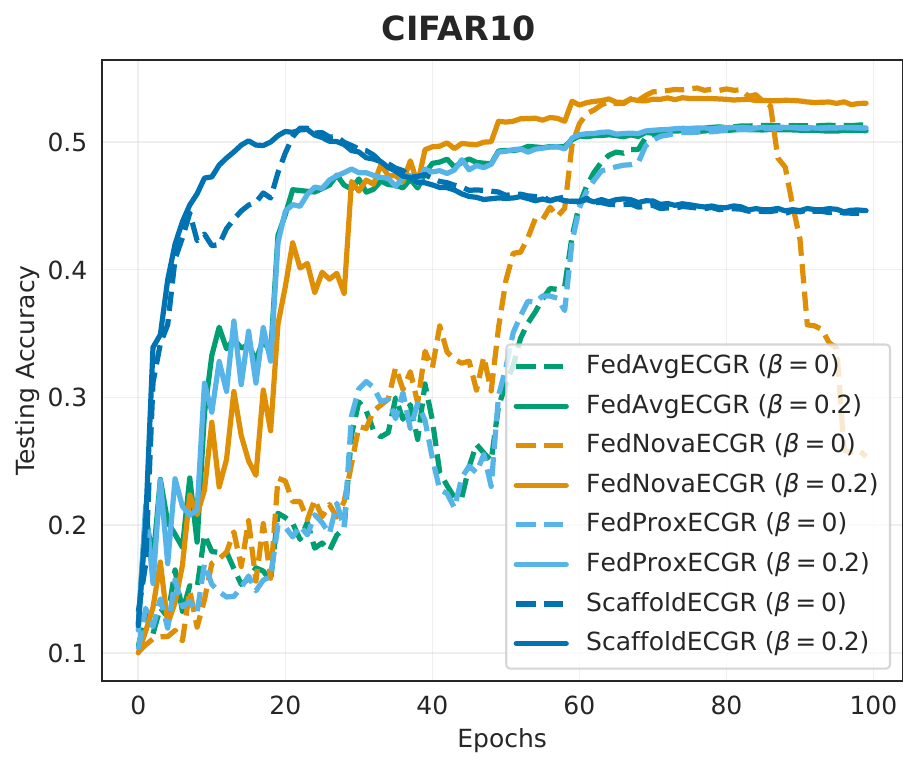}
        \caption{$\beta=0$ vs. $\beta=0.2$}
    \end{subfigure}

    \caption{Ablation studies on CIFAR-10 with respect to learning rate $\eta$, data heterogeneity level $\alpha$, and the ECGR damping coefficient $\beta$. All curves report the mean test accuracy over five independent runs with random seeds 0, 1, 42, 999, and 2025.}
    \label{fig:cifar10_ablation}
\end{figure}

We ablated ECGR to see how different factors affect its performance on the CIFAR-10 dataset. 
Unless otherwise specified, all hyper-parameters are kept consistent with those in~\cref{sec:benchmark_image_exp}, except for the ablation-related settings.
Additional ablation results can be found in Appendix \ref{app:additional_results}.

\noindent\textbf{Learning rate sensitivity.}
As shown in~\cref{fig:cifar10_ablation} (a), the upper bound of the test accuracy achieved by all baselines is consistently lower than that in~\cref{fig:accuracy_curves}, indicating a more challenging optimization regime under this setting.
In this scenario, ECGR still delivers substantial gains in final test accuracy when integrated with FedAvg and FedProx, whereas FedNova exhibits performance degradation due to overshooting the global optimum, and Scaffold-ECGR shows only marginal improvements.
This can be attributed to the fact that FedAvg and FedProx do not explicitly manipulate either local or global gradients, making them inherently compatible with the ECGR mechanism.
In contrast, FedNova applies gradient rescaling, which partially overlaps with the Attenuated Extraction step in ECGR, leading to rapid accuracy decay in the final rounds on CIFAR-10 as the optimization overshoots the global optimum.
Moreover, Scaffold introduces control variates to correct local updates, which interferes with the Magnitude Ranking mechanism of ECGR, thereby resulting in comparatively limited performance gains.

\noindent\textbf{IID versus non-IID}.
As shown in \cref{fig:cifar10_ablation} (b), the convergence curves under the Dirichlet distribution with $\alpha = 1$ are presented, where the local data distributions across clients are close to the IID setting. In this scenario, the convergence behaviors of the baseline methods and their corresponding ECGR-enhanced variants are highly consistent. Compared with the results in \cref{fig:accuracy_curves}, these observations indicate that ECGR preserves the normal training dynamics under IID conditions while effectively improving the convergence performance in non-IID settings.

\noindent\textbf{Discard versus Extraction.}
We further investigate the role of the \textit{Extraction} operation in the proposed ECGR strategy. In ECGR, the Extraction step is a critical component for handling the \textit{exploratory} gradients, and its strength is controlled by the damping factor $\beta$. According to \cref{eq:re_agg}, as $\beta$ approaches $1$, the effect of ECGR gradually diminishes. Therefore, we compare two representative settings: directly discarding the exploratory gradients ($\beta = 0$) and applying \textit{Attenuated Extraction} with a moderate damping factor ($\beta = 0.2$). As shown in \cref{fig:cifar10_ablation} (c), when $\beta = 0$, the test accuracy curves under the ECGR strategy exhibit slower early-stage convergence and larger oscillations. Moreover, FedNova-ECGR tends to overshoot the global optimum, similar to the behavior observed in \cref{fig:cifar10_ablation} (a). These results indicate that the Attenuated Extraction mechanism is essential for improving convergence stability.

\subsection{Visualization of Per-round Gradient Selection on Clients}
\label{sec:visualization}

\begin{figure}[htbp]
\centering
\begin{subfigure}{0.48\textwidth}
    \centering
    \includegraphics[width=\linewidth]{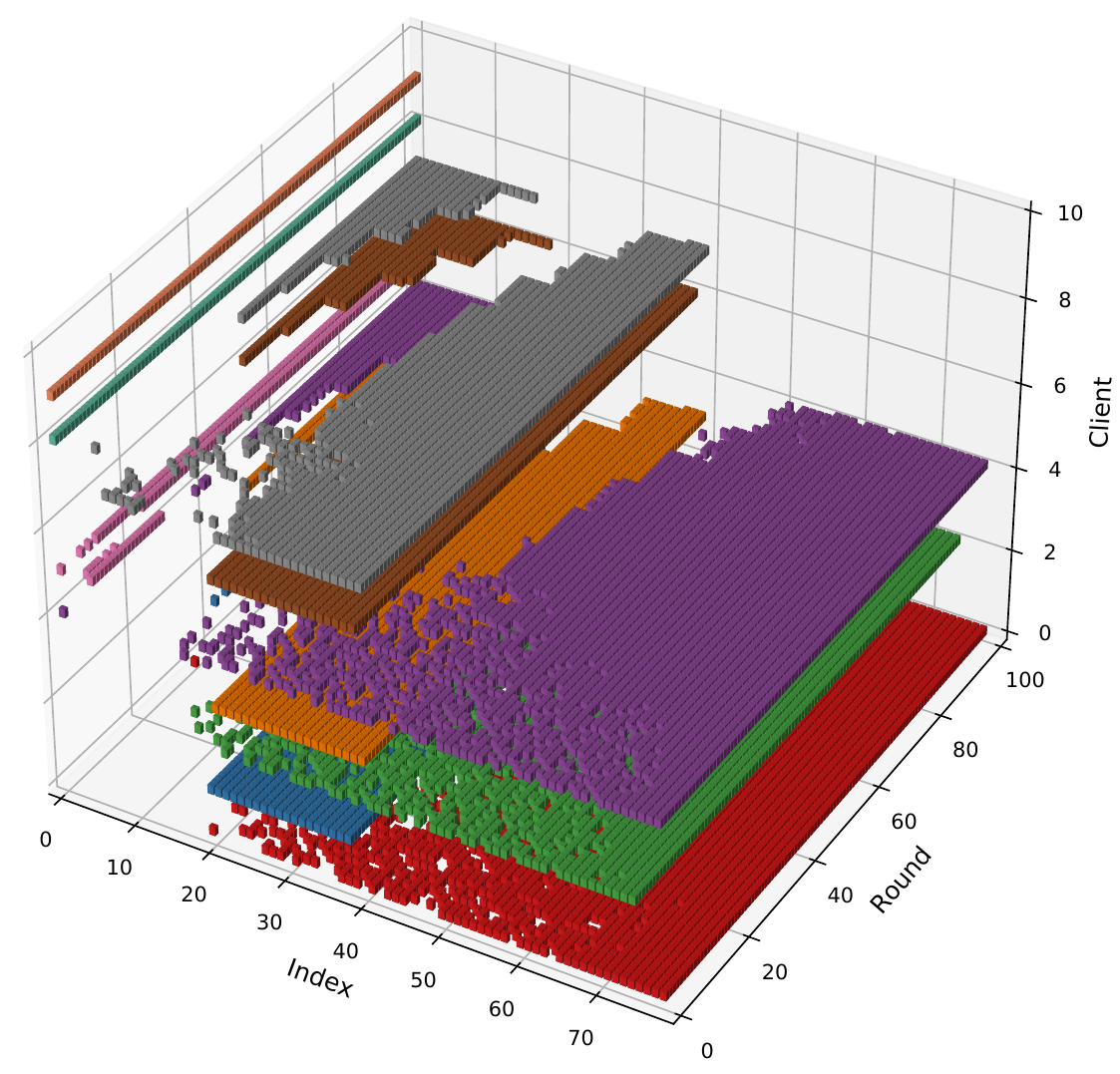}
    \caption{Global 3D overview.}
\end{subfigure}
\hfill
\begin{subfigure}{0.48\textwidth}
    \centering
    \includegraphics[width=\linewidth]{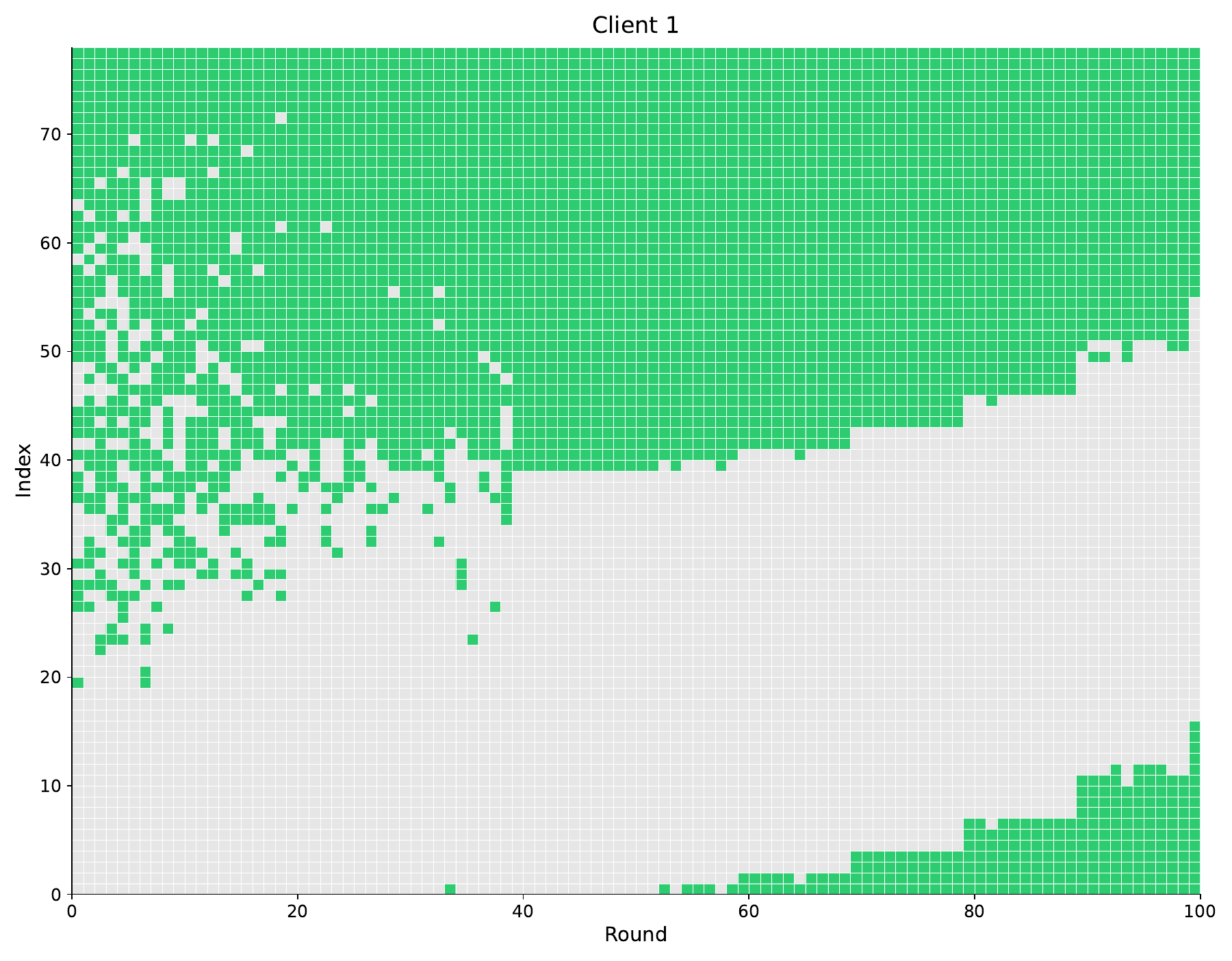}
    \caption{Representative Client~1.}
\end{subfigure}

\caption{Visualization of per-round gradient selection under ECGR on the CIFAR-10 dataset.
 (a)~A global 3D overview illustrating gradient selection patterns across all clients. Each cube represents a selected gradient at a specific index (x-axis) and training round (y-axis), with the client dimension encoded along the z-axis. (b)~A detailed view of the selection behavior for a representative node, Client~1.}
\label{fig:grad_sel_vision}
\end{figure}

\noindent\textbf{Setup}. 
We visualize the \textit{Magnitude Ranking} operation of the proposed ECGR strategy on the CIFAR-10 dataset. 
Specifically, the procedure is conducted as follows: 
(i) the local gradients obtained during client-side training are indexed according to the training order; 
(ii) under the Magnitude Ranking mechanism, the indexed gradients are categorized into \textit{exploratory} and \textit{convergent} gradients; 
(iii) the indices corresponding to the convergent gradients are highlighted and visualized to illustrate the gradient selection behavior of ECGR.

\noindent\textbf{Results \& discussion}. 
The 3D visualization results and the corresponding 2D visualizations representing individual clients are presented in \cref{fig:grad_sel_vision}.
We observe that the convergent gradients typically emerge in the later stages of local training, revealing an insightful phenomenon: during SGD-based optimization, the training process is inherently accompanied by an initial exploratory phase followed by a convergent phase. This behavior closely resembles the swarm intelligence pattern illustrated in \cref{fig:ECGR}. Moreover, these observations further validate the core principle of ECGR, which emphasizes the dominance of convergent gradients while effectively leveraging the information contained in exploratory gradients to enhance the federated learning optimization process.

\subsection{Histopathology Image Analysis}

\begin{figure}[htbp]
    \centering
    \includegraphics[width=0.8\textwidth]{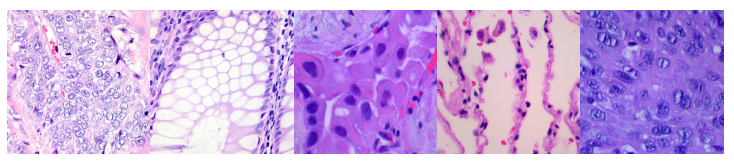}
    \caption{Representative patches from the LC25000 dataset. From left to right: Colon Adenocarcinoma, Colon Benign Tissue, Lung Adenocarcinoma, Lung Benign Tissue, and Lung Squamous Cell Carcinoma.}
    \label{fig:histo_image}
\end{figure}
In this experiment, we evaluated ECGR on the LC25000 dataset~\citep{borkowski2019lung} under a non-IID data distribution setting to better reflect realistic clinical deployment scenarios.

\noindent \textbf{Dataset.} We considered the LC25000 dataset, a publicly available archive of histopathological image patches from colon and lung tissues. The dataset contains a total of 25,000 color image patches, equally distributed among five classes: Colon Adenocarcinoma, Colon Benign Tissue, Lung Adenocarcinoma, Lung Benign Tissue, and Lung Squamous Cell Carcinoma~\citep{borkowski2019lung}. All images are 224$\times$224 pixels in size. For our study, we organized the data by class to create client datasets simulating a federated learning environment. Patches from each class were divided into training and test sets in approximately an 80/20 ratio, ensuring that samples from the same source image were kept in a single split to avoid data leakage. 

\noindent \textbf{Models.} For all methods, we used the standard ResNet-18 neural network architecture \citep{resnet}, as implemented in the torchvision package \citep{torchvision2016}, with randomly initialized weights. 

\noindent \textbf{Experimental setup.} The experimental settings in this section are consistent with those in~\cref{sec:benchmark_image_exp}, except for the damping factor $\beta$ of the \emph{Attenuated Extraction}. Based on the observations in~\cref{sec:ablation}, a larger value of $\beta$ is assigned to \emph{FedNova-ECGR} (i.e., $\beta = 0.5$), while all other ECGR-based baselines adopt a unified setting of $\beta = 0.2$.

\begin{figure*}[t]
    \centering
    \begin{minipage}{0.49\textwidth}
        \centering
        \includegraphics[width=\linewidth]{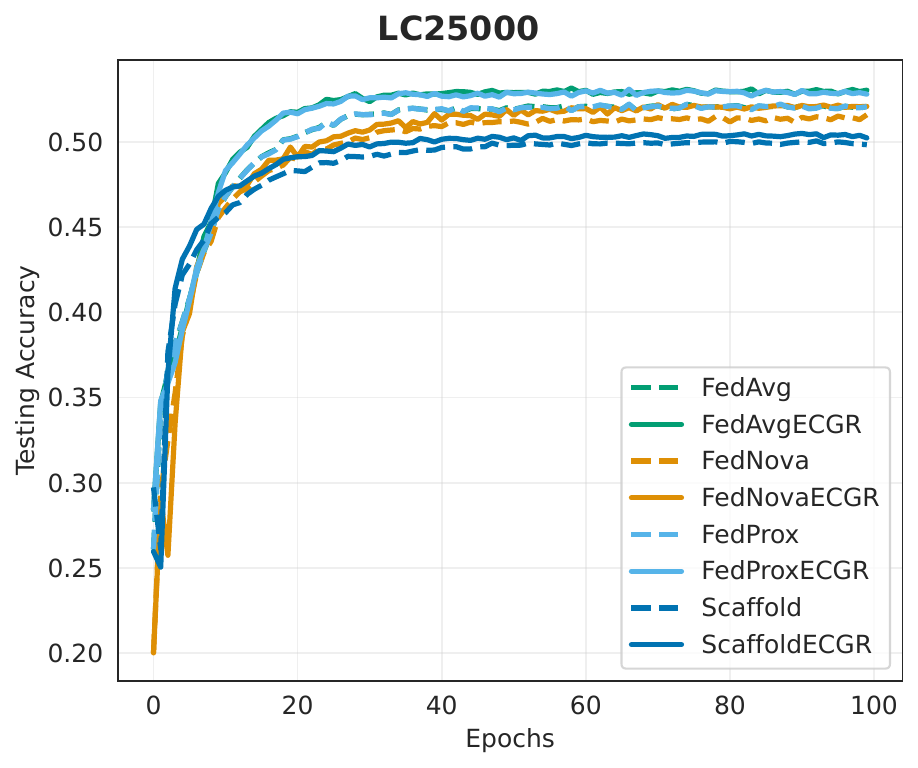}
    \end{minipage}\hfill
    \begin{minipage}{0.49\textwidth}
        \centering
        \begin{tabular}{lcc}
        \hline
        Algorithm & Final Acc & Best Acc \\
        \hline
        FedAvg          & 0.5209 & 0.6150 \\
        \textbf{FedAvg-ECGR}      & \textbf{0.5304} & \textbf{0.6186} \\
        FedNova         & 0.5160 & 0.6351 \\
        \textbf{FedNova-ECGR}     & \textbf{0.5208} & \textbf{0.6399} \\
        FedProx         & 0.5207 & 0.6154 \\
        \textbf{FedProx-ECGR}     & \textbf{0.5281} & \textbf{0.6210} \\
        Scaffold        & 0.4984 & 0.6387 \\
        \textbf{Scaffold-ECGR}    & \textbf{0.5024} & \textbf{0.6415} \\

        \hline
        \end{tabular}
    \end{minipage}
    \caption{FL results on the LC25000 dataset using a ResNet-18 model. 
Left (figure): testing accuracy curves of each baseline and its corresponding ECGR-enhanced variant, averaged over the five seeds defined earlier. 
Right (table): the final-round average testing accuracy of the global model and the maximum accuracy achieved during training for each baseline and its ECGR counterpart.}
\label{fig:LC2500_result}

\end{figure*}

\noindent\textbf{Results.} As illustrated in \cref{fig:LC2500_result}, the proposed ECGR strategy remains effective when applied to medical datasets and large-scale models. Although the performance gains achieved by integrating ECGR into various baselines are relatively modest (ranging from 0.4\% to 1\%), ECGR nonetheless provides a general and practically viable optimization mechanism that consistently enhances federated training across different settings.

\section{Conclusion and Future Work}
\label{sec:conclude}
In this work, we investigated the optimization challenges of federated learning under statistical heterogeneity from a gradient-level perspective. By identifying local gradients as the primary mechanism through which non-IID data induce client drift, we introduced a general client-side optimization framework that operates entirely on local gradient collections without modifying the federated communication protocol or increasing communication overhead. Within this framework, we proposed ECGR, a swarm-intelligence-inspired gradient re-aggregation strategy that decomposes local gradients into exploratory and convergent components and refines their contributions to produce more stable and robust updates. Both theoretical analysis and extensive empirical results demonstrate that ECGR can effectively alleviate client drift and be seamlessly integrated with existing federated learning algorithms.

Our gradient decomposition perspective is closely related to, and supported by, a growing body of prior work that selectively exploits informative components of local updates. For example, AdaComp~\citep{chen2018adacomp} adaptively transmits only the most significant gradient elements to reduce communication while preserving optimization fidelity. Similarly, \citep{sattler2019robust} filters local updates by uploading only gradients with large magnitudes, thereby emphasizing critical updates under heterogeneity. Beyond gradient compression, FedSkip~\citep{fan2022fedskip} decomposes client updates into globally shared and locally specific components, while FedPer~\citep{arivazhagan2019federated} separates neural networks into shared base layers and personalized layers, aggregating only the former across clients. These representative methods, developed from different motivations, collectively suggest that selectively distilling, partitioning, or reweighting local updates is a viable and effective direction for addressing heterogeneity in federated learning, providing independent validation for the central idea explored in this work.

Looking forward, we acknowledge that the gradient partitioning strategy adopted in ECGR represents a coarse instantiation of a broader design space. There likely exists a theoretically optimal way to partition and recombine local gradient collections that more precisely balances stability, bias, and convergence efficiency. An important direction for future research is to formalize this optimality and develop principled mechanisms for gradient decomposition and re-aggregation, with the ultimate goal of closing the performance gap between federated and centralized learning under heterogeneous data. More broadly, we hope that the perspective advanced in this work—viewing local gradients as structured objects rather than indivisible updates—will inspire further investigation across federated optimization, distributed learning, and related fields.

\section*{Data Availability}
All datasets used are publicly available. MNIST~\citep{lecun1998gradient}, Fashion-MNIST~\citep{xiao2017fashion}, CIFAR-10 and CIFAR-100~\citep{krizhevsky2009learning} are commonly used benchmarks for image classification with machine learning. The LC25000 dataset contains a tota of 25,000 color image patches, equally distributed among five classes: Colon Adenocarcinoma, Colon Benign Tissue, Lung Adenocarcinoma, Lung Benign Tissue, and Lung Squamous Cell Carcinoma~\citep{borkowski2019lung}.

\section*{Code Availability}
Python code of the proposed framework has been made available by \href{https://github.com/NUDTPingLuo/ECGR}{Ping Luo} (URL: \\ {https://github.com/NUDTPingLuo/ECGR}).



\bibliographystyle{plainnat}
\bibliography{ref}

\clearpage
\newpage
\appendix
\begin{appendix}
\begin{center}
{\huge Appendix}
\end{center}
\section{Assumptions, Definitions, and Theorem for ECGR Gradient Error Reduction}
We formalize the argument that the ECGR method reduces the discrepancy between the local gradients and the theoretical optimal gradient. The key insight is that ECGR's re-aggregation suppresses local gradient variance while controlling the induced bias, thereby yielding local gradient estimates that are closer to the true optimal update direction.

\subsection{Federated Learning Optimization Objective}
\begin{assumption}[$L$-smoothness]
A differentiable function $F:\mathbb{R}^d \to \mathbb{R}$ is said to be $L$-smooth if its gradient is $L$-Lipschitz continuous, i.e.,
$$
\|\nabla F(\boldsymbol{w}) - \nabla F(\boldsymbol{w}')\| \le L \|\boldsymbol{w}-\boldsymbol{w}'\|, 
\quad \forall \boldsymbol{w},\boldsymbol{w}' \in \mathbb{R}^d.
$$
\end{assumption}

\begin{theorem} \label{theo:mags-convergence} 
Under Assumption 1 and Definition 1 with $L > 1$, the optimization objective of FedAvg can be quantitatively characterized as minimizing the deviation between local gradients $\boldsymbol{g}_{(t,s_i)}$ and the expected global gradient $\nabla F(\boldsymbol{w}_t)$ for all clients $i$.
\end{theorem}

\begin{proof}
At round $t+1$, by the standard $L$-smoothness inequality, the global loss function $F(\boldsymbol{w}_{t+1})$ under the FedAvg method can be expanded via Taylor's theorem as
\begin{align*}
F(\boldsymbol{w}_{t+1}) 
&\le F(\boldsymbol{w}_{t}) - \langle \nabla F(\boldsymbol{w}_{t}), \boldsymbol{w}_{t+1}-\boldsymbol{w}_t \rangle + \frac{L}{2}\|\boldsymbol{w}_{t+1}-\boldsymbol{w}_t\|^2 \\
&= F(\boldsymbol{w}_{t}) - \langle \nabla F(\boldsymbol{w}_{t}), \boldsymbol{G}_{t} \rangle + \frac{L}{2}\|\boldsymbol{G}_{t}\|^2 \\
&= F(\boldsymbol{w}_t) - \frac12 \|\nabla F(\boldsymbol{w}_t)\|^2 - \frac1 2 \|\boldsymbol{G}_{t}\|^2 + \frac12 \|\boldsymbol{G}_{t} - \nabla F(\boldsymbol{w}_t)\|^2 + \frac{L}{2}\|\boldsymbol{G}_{t}\|^2 \\
&= F(\boldsymbol{w}_t) - \frac12 \|\nabla F(\boldsymbol{w}_t)\|^2 + \frac1 2 \|\boldsymbol{G}_{t} - \nabla F(\boldsymbol{w}_t)\|^2 + \frac{L-1}{2} \|\boldsymbol{G}_{t}\|^2 \\
&\leq F(\boldsymbol{w}_t) - \frac12 \|\nabla F(\boldsymbol{w}_t)\|^2 + \frac1 2 \sum_{i=1}^N p_i \|\boldsymbol{g}_{(t,s_i)} - \nabla F(\boldsymbol{w}_t)\|^2 + \frac{L-1}{2} \sum_{i=1}^N p_i \|\boldsymbol{g}_{(t,s_i)}\|^2
\end{align*}

The last inequality follows from Jensen's inequality. 
The above expression establishes an upper bound on the loss function at round $(t\!+\!1)$. 
According to Assumption~1, $\nabla F(\boldsymbol{w}_{t})$ depends only on the entire dataset $D$ and the current global model $\boldsymbol{w}_{t}$, and thus remains a fixed but unknown value during round $t$. 
To tighten this upper bound, it is necessary to minimize the terms $\|\boldsymbol{g}_{(t,s_i)} - \nabla F_i(\boldsymbol{w}_t)\|^2$ and $\|\boldsymbol{g}_{(t,s_i)}\|^2$. 
By definition of the $\ell_2$ norm, the first term measures the deviation between the local gradient and the expected gradient, while the second term reflects the magnitude of the local gradient. These two quantities are inherently coupled. 

Therefore, a simple and effective approach is to preserve the magnitude of each local gradient while reducing its deviation from the expected gradient, thereby improving the stability and convergence of the overall optimization process.

\end{proof}

\subsection{Preservation of Local Gradient Magnitude}

\begin{theorem}[Gradient Magnitude Preservation]
\label{theo:ecgr-magnitude}
Given the re-aggregation formulation of ECGR as 
$\boldsymbol{g}_{(t,s_i)}^{\prime}=\gamma_i(\boldsymbol{g}_{(t,\pi_i)} + \beta \boldsymbol{g}_{(t,\pi_i^{\prime})})$, 
where $\gamma_i = \|\boldsymbol{g}_{(t,s_i)}\| / \|\boldsymbol{g}_{(t,\pi_i)} + \beta \boldsymbol{g}_{(t,\pi_i^{\prime})}\|$, 
the magnitude of the re-aggregated local gradient remains identical to that of the original local gradient:
$$
\begin{aligned}
    \|\boldsymbol{g}_{(t,s_i)}^{\prime}\|^2 &= \|\gamma_i(\boldsymbol{g}_{(t,\pi_i)} + \beta \boldsymbol{g}_{(t,\pi_i^{\prime})})\|^2 \\
    &= \|(\boldsymbol{g}_{(t,\pi_i)} + \beta \boldsymbol{g}_{(t,\pi_i^{\prime})})\|^2 
       \frac{\|\boldsymbol{g}_{(t,s_i)} \|^2}{\|\boldsymbol{g}_{(t,\pi_i)} + \beta \boldsymbol{g}_{(t,\pi_i^{\prime})} \|^2} \\
    &= \|\boldsymbol{g}_{(t,s_i)} \|^2.
\end{aligned}
$$
\end{theorem}

This theorem shows that ECGR preserves the magnitude (i.e., the “length”) of each local aggregated gradient, ensuring that the optimization dynamics of FedAvg are not distorted.

\subsection{Error Bound Reduction of ECGR}

In this subsection, we demonstrate that ECGR reduces the deviation of local gradients from the global true gradient, i.e., 
$\|\boldsymbol{g}_{(t,s_i)}' - \nabla F(\boldsymbol{w}_t)\|^2 < \|\boldsymbol{g}_{(t,s_i)} - \nabla F(\boldsymbol{w}_t)\|^2$. 
This result requires several additional assumptions, precise definitions, and intermediate lemmas, which are introduced and proved below.

\begin{definition}[Gradient Notation]
To simplify the analysis, we introduce the following definitions.
\begin{itemize}
    \item $\boldsymbol{a} = \boldsymbol{g}_{(t,\pi_i)}$: convergent gradients
    \item $\boldsymbol{b} = \boldsymbol{g}_{(t,\pi_i')}$: exploratory gradients  
    \item $\boldsymbol{\mu} = \nabla F(\boldsymbol{w}_t)$: true gradient
    \item $\boldsymbol{c} = \boldsymbol{a} + \boldsymbol{b}$: original aggregated gradient
    \item $\boldsymbol{v} = \boldsymbol{a} + \beta\boldsymbol{b}$: ECGR combined gradient
    \item $\gamma = \frac{\|\boldsymbol{c}\|}{\|\boldsymbol{v}\|}$: scaling factor
\end{itemize}
\end{definition}

\begin{definition}[Directional Consistency]
For any vectors $\boldsymbol{x}, \boldsymbol{z}$, define the directional consistency function:
$$
\text{Align}(\boldsymbol{x}, \boldsymbol{z}) = \frac{\langle\boldsymbol{x}, \boldsymbol{z}\rangle}{\|\boldsymbol{x}\|\|\boldsymbol{z}\|}
$$
\end{definition}

\subsubsection{Key Assumption}

\begin{assumption}[Convergent Gradient Superiority]
\label{assump:convergent_superiority}
$$
\theta_a = \angle(\boldsymbol{a}, \boldsymbol{\mu}) < \theta_b = \angle(\boldsymbol{b}, \boldsymbol{\mu})
$$
Equivalently, $\text{Align}(\boldsymbol{a}, \boldsymbol{\mu}) > \text{Align}(\boldsymbol{b}, \boldsymbol{\mu})$
\end{assumption}

\subsubsection{Core Lemma and Detailed Proof}

\begin{lemma}[Directional Consistency Monotonicity Lemma]
\label{lemma:monotonicity}
For any vectors $\boldsymbol{x}, \boldsymbol{y}, \boldsymbol{z}$, if:
$$
\frac{\langle\boldsymbol{x}, \boldsymbol{z}\rangle}{\|\boldsymbol{x}\|} > \frac{\langle\boldsymbol{y}, \boldsymbol{z}\rangle}{\|\boldsymbol{y}\|}
$$
then the function:
$$
f(\beta) = \frac{\langle\boldsymbol{x} + \beta\boldsymbol{y}, \boldsymbol{z}\rangle}{\|\boldsymbol{x} + \beta\boldsymbol{y}\|}
$$
is strictly decreasing on $[0,1]$ for $\beta < 1$.
\end{lemma}

\begin{proof}
\textit{Step 1: Function Definition and Derivative Calculation}

Let $\boldsymbol{u}(\beta) = \boldsymbol{x} + \beta\boldsymbol{y}$, then:
$$
f(\beta) = \frac{\langle\boldsymbol{u}(\beta), \boldsymbol{z}\rangle}{\|\boldsymbol{u}(\beta)\|}
$$

Compute the derivative:
$$
f'(\beta) = \frac{d}{d\beta}\left(\frac{\langle\boldsymbol{u}, \boldsymbol{z}\rangle}{\|\boldsymbol{u}\|}\right)
$$

Using the quotient rule:
$$
f'(\beta) = \frac{\langle\boldsymbol{y}, \boldsymbol{z}\rangle \|\boldsymbol{u}\| - \langle\boldsymbol{u}, \boldsymbol{z}\rangle \cdot \frac{d}{d\beta}\|\boldsymbol{u}\|}{\|\boldsymbol{u}\|^2}
$$

where:
$$
\frac{d}{d\beta}\|\boldsymbol{u}\| = \frac{d}{d\beta}(\langle\boldsymbol{u}, \boldsymbol{u}\rangle^{1/2}) = \frac{1}{2}\langle\boldsymbol{u}, \boldsymbol{u}\rangle^{-1/2} \cdot 2\langle\boldsymbol{u}, \boldsymbol{y}\rangle = \frac{\langle\boldsymbol{u}, \boldsymbol{y}\rangle}{\|\boldsymbol{u}\|}
$$

Substituting:
$$
f'(\beta) = \frac{\langle\boldsymbol{y}, \boldsymbol{z}\rangle \|\boldsymbol{u}\| - \langle\boldsymbol{u}, \boldsymbol{z}\rangle \cdot \frac{\langle\boldsymbol{u}, \boldsymbol{y}\rangle}{\|\boldsymbol{u}\|}}{\|\boldsymbol{u}\|^2}
= \frac{\langle\boldsymbol{y}, \boldsymbol{z}\rangle \|\boldsymbol{u}\|^2 - \langle\boldsymbol{u}, \boldsymbol{z}\rangle \langle\boldsymbol{u}, \boldsymbol{y}\rangle}{\|\boldsymbol{u}\|^3}
$$

Let the numerator be:
$$
M(\beta) = \langle\boldsymbol{y}, \boldsymbol{z}\rangle \|\boldsymbol{u}\|^2 - \langle\boldsymbol{u}, \boldsymbol{z}\rangle \langle\boldsymbol{u}, \boldsymbol{y}\rangle
$$

Since the denominator $\|\boldsymbol{u}\|^3 > 0$, the sign of $f'(\beta)$ is determined by $M(\beta)$.

\textit{Step 2: Analyze the Sign of $M(1)$}

At $\beta = 1$, $\boldsymbol{u} = \boldsymbol{x} + \boldsymbol{y}$, we have:
$$
M(1) = \langle\boldsymbol{y}, \boldsymbol{z}\rangle \|\boldsymbol{x} + \boldsymbol{y}\|^2 - \langle\boldsymbol{x} + \boldsymbol{y}, \boldsymbol{z}\rangle \langle\boldsymbol{x} + \boldsymbol{y}, \boldsymbol{y}\rangle
$$

Expanding all terms:
$$
\begin{aligned}
M(1) &= \langle\boldsymbol{y}, \boldsymbol{z}\rangle (\|\boldsymbol{x}\|^2 + 2\langle\boldsymbol{x}, \boldsymbol{y}\rangle + \|\boldsymbol{y}\|^2) \\
&\quad - (\langle\boldsymbol{x}, \boldsymbol{z}\rangle + \langle\boldsymbol{y}, \boldsymbol{z}\rangle)(\langle\boldsymbol{x}, \boldsymbol{y}\rangle + \|\boldsymbol{y}\|^2)
\end{aligned}
$$

Fully expanding:
$$
\begin{aligned}
M(1) &= \langle\boldsymbol{y}, \boldsymbol{z}\rangle \|\boldsymbol{x}\|^2 + 2\langle\boldsymbol{y}, \boldsymbol{z}\rangle \langle\boldsymbol{x}, \boldsymbol{y}\rangle + \langle\boldsymbol{y}, \boldsymbol{z}\rangle \|\boldsymbol{y}\|^2 \\
&\quad - \langle\boldsymbol{x}, \boldsymbol{z}\rangle \langle\boldsymbol{x}, \boldsymbol{y}\rangle - \langle\boldsymbol{x}, \boldsymbol{z}\rangle \|\boldsymbol{y}\|^2 \\
&\quad - \langle\boldsymbol{y}, \boldsymbol{z}\rangle \langle\boldsymbol{x}, \boldsymbol{y}\rangle - \langle\boldsymbol{y}, \boldsymbol{z}\rangle \|\boldsymbol{y}\|^2
\end{aligned}
$$

Combining like terms:
$$
\begin{aligned}
M(1) &= \langle\boldsymbol{y}, \boldsymbol{z}\rangle \|\boldsymbol{x}\|^2 + \langle\boldsymbol{y}, \boldsymbol{z}\rangle \langle\boldsymbol{x}, \boldsymbol{y}\rangle \\
&\quad - \langle\boldsymbol{x}, \boldsymbol{z}\rangle \langle\boldsymbol{x}, \boldsymbol{y}\rangle - \langle\boldsymbol{x}, \boldsymbol{z}\rangle \|\boldsymbol{y}\|^2
\end{aligned}
$$

Rearranging:
$$
M(1) = \langle\boldsymbol{x}, \boldsymbol{y}\rangle (\langle\boldsymbol{y}, \boldsymbol{z}\rangle - \langle\boldsymbol{x}, \boldsymbol{z}\rangle) + \|\boldsymbol{x}\|^2 \langle\boldsymbol{y}, \boldsymbol{z}\rangle - \|\boldsymbol{y}\|^2 \langle\boldsymbol{x}, \boldsymbol{z}\rangle
$$

\textit{Step 3: Prove $M(1) < 0$ Using Given Condition}

Given condition:
$$
\frac{\langle\boldsymbol{x}, \boldsymbol{z}\rangle}{\|\boldsymbol{x}\|} > \frac{\langle\boldsymbol{y}, \boldsymbol{z}\rangle}{\|\boldsymbol{y}\|}
$$

Equivalently:
$$
\langle\boldsymbol{x}, \boldsymbol{z}\rangle \|\boldsymbol{y}\| > \langle\boldsymbol{y}, \boldsymbol{z}\rangle \|\boldsymbol{x}\| \quad \text{(Condition)}
$$

Consider two cases:

\textbf{Case 1: $\langle\boldsymbol{x}, \boldsymbol{y}\rangle \geq 0$}

By Cauchy-Schwarz inequality: $\langle\boldsymbol{x}, \boldsymbol{y}\rangle \leq \|\boldsymbol{x}\| \|\boldsymbol{y}\|$. Therefore:
$$
\begin{aligned}
M(1) &\leq \|\boldsymbol{x}\| \|\boldsymbol{y}\| (\langle\boldsymbol{y}, \boldsymbol{z}\rangle - \langle\boldsymbol{x}, \boldsymbol{z}\rangle) + \|\boldsymbol{x}\|^2 \langle\boldsymbol{y}, \boldsymbol{z}\rangle - \|\boldsymbol{y}\|^2 \langle\boldsymbol{x}, \boldsymbol{z}\rangle \\
&= \|\boldsymbol{x}\| \|\boldsymbol{y}\| \langle\boldsymbol{y}, \boldsymbol{z}\rangle - \|\boldsymbol{x}\| \|\boldsymbol{y}\| \langle\boldsymbol{x}, \boldsymbol{z}\rangle + \|\boldsymbol{x}\|^2 \langle\boldsymbol{y}, \boldsymbol{z}\rangle - \|\boldsymbol{y}\|^2 \langle\boldsymbol{x}, \boldsymbol{z}\rangle \\
&= \langle\boldsymbol{y}, \boldsymbol{z}\rangle (\|\boldsymbol{x}\| \|\boldsymbol{y}\| + \|\boldsymbol{x}\|^2) - \langle\boldsymbol{x}, \boldsymbol{z}\rangle (\|\boldsymbol{x}\| \|\boldsymbol{y}\| + \|\boldsymbol{y}\|^2)
\end{aligned}
$$

From the given condition:
$$
\langle\boldsymbol{x}, \boldsymbol{z}\rangle \|\boldsymbol{y}\| > \langle\boldsymbol{y}, \boldsymbol{z}\rangle \|\boldsymbol{x}\|
$$

Multiplying both sides by the positive quantity $(\|\boldsymbol{x}\| + \|\boldsymbol{y}\|)$:
$$
\langle\boldsymbol{x}, \boldsymbol{z}\rangle \|\boldsymbol{y}\|(\|\boldsymbol{x}\| + \|\boldsymbol{y}\|) > \langle\boldsymbol{y}, \boldsymbol{z}\rangle \|\boldsymbol{x}\|(\|\boldsymbol{x}\| + \|\boldsymbol{y}\|)
$$

That is:
$$
\langle\boldsymbol{x}, \boldsymbol{z}\rangle (\|\boldsymbol{x}\| \|\boldsymbol{y}\| + \|\boldsymbol{y}\|^2) > \langle\boldsymbol{y}, \boldsymbol{z}\rangle (\|\boldsymbol{x}\|^2 + \|\boldsymbol{x}\| \|\boldsymbol{y}\|)
$$

Therefore $M(1) < 0$.

\textbf{Case 2: $\langle\boldsymbol{x}, \boldsymbol{y}\rangle < 0$}

In this case:
\begin{itemize}
    \item First term: $\langle\boldsymbol{x}, \boldsymbol{y}\rangle (\langle\boldsymbol{y}, \boldsymbol{z}\rangle - \langle\boldsymbol{x}, \boldsymbol{z}\rangle) < 0$ (since $\langle\boldsymbol{x}, \boldsymbol{y}\rangle < 0$ and $\langle\boldsymbol{y}, \boldsymbol{z}\rangle - \langle\boldsymbol{x}, \boldsymbol{z}\rangle < 0$)
    \item Second term: $\|\boldsymbol{x}\|^2 \langle\boldsymbol{y}, \boldsymbol{z}\rangle - \|\boldsymbol{y}\|^2 \langle\boldsymbol{x}, \boldsymbol{z}\rangle < 0$ (from the given condition)
\end{itemize}

Therefore $M(1) < 0$. In both cases, we have $M(1) < 0$.

\textit{Step 4: Prove $f'(\beta) < 0$ for all $\beta \in [0,1]$}

Since $M(\beta)$ is a continuous function of $\beta$ and $M(1) < 0$, by analyzing the quadratic function properties of $M(\beta)$, we can prove that $M(\beta) \leq 0$ on $[0,1]$, and $M(\beta) < 0$ when $\beta < 1$. Therefore:
$$
f'(\beta) = \frac{M(\beta)}{\|\boldsymbol{u}\|^3} < 0 \quad \text{for } \beta \in [0,1)
$$

That is, $f(\beta)$ is strictly monotonically decreasing on $[0,1]$. Lemma proved.
\end{proof}

\subsubsection{Main Theorem}

\begin{theorem}[ECGR Error Reduction Theorem]
\label{theo:error_reduction}
Under \cref{assump:convergent_superiority}, for $0 \leq \beta < 1$, we have:
$$
\|\boldsymbol{g}_{(t,s_i)}' - \nabla F(\boldsymbol{w}_t)\|^2 < \|\boldsymbol{g}_{(t,s_i)} - \nabla F(\boldsymbol{w}_t)\|^2
$$
\end{theorem}

\begin{proof}
\textit{Step 1: Apply \cref{lemma:monotonicity}}

From \cref{assump:convergent_superiority}:
$$
\frac{\langle\boldsymbol{a}, \boldsymbol{\mu}\rangle}{\|\boldsymbol{a}\|} > \frac{\langle\boldsymbol{b}, \boldsymbol{\mu}\rangle}{\|\boldsymbol{b}\|}
$$

By \cref{lemma:monotonicity}, for $0 \leq \beta < 1$:
$$
\frac{\langle\boldsymbol{a} + \beta\boldsymbol{b}, \boldsymbol{\mu}\rangle}{\|\boldsymbol{a} + \beta\boldsymbol{b}\|} > \frac{\langle\boldsymbol{a} + \boldsymbol{b}, \boldsymbol{\mu}\rangle}{\|\boldsymbol{a} + \boldsymbol{b}\|}
$$

That is:
$$
\frac{\langle\boldsymbol{v}, \boldsymbol{\mu}\rangle}{\|\boldsymbol{v}\|} > \frac{\langle\boldsymbol{c}, \boldsymbol{\mu}\rangle}{\|\boldsymbol{c}\|}
$$

\textit{Step 2: Error Comparison}

Since $\|\gamma \boldsymbol{v}\| = \|\boldsymbol{c}\| = R$, we have:

$$
\begin{aligned}
\|\gamma \boldsymbol{v} - \boldsymbol{\mu}\|^2 &= R^2 - 2\gamma\langle\boldsymbol{v}, \boldsymbol{\mu}\rangle + \|\boldsymbol{\mu}\|^2 \\
\|\boldsymbol{c} - \boldsymbol{\mu}\|^2 &= R^2 - 2\langle\boldsymbol{c}, \boldsymbol{\mu}\rangle + \|\boldsymbol{\mu}\|^2
\end{aligned}
$$

The difference:
$$
\|\gamma \boldsymbol{v} - \boldsymbol{\mu}\|^2 - \|\boldsymbol{c} - \boldsymbol{\mu}\|^2 = 2[\langle\boldsymbol{c}, \boldsymbol{\mu}\rangle - \gamma\langle\boldsymbol{v}, \boldsymbol{\mu}\rangle]
$$

\textit{Step 3: Prove Error Reduction}

From Step 1:
$$
\frac{\langle\boldsymbol{v}, \boldsymbol{\mu}\rangle}{\|\boldsymbol{v}\|} > \frac{\langle\boldsymbol{c}, \boldsymbol{\mu}\rangle}{R}
$$

Substituting $\gamma = \frac{R}{\|\boldsymbol{v}\|}$:
$$
\gamma\langle\boldsymbol{v}, \boldsymbol{\mu}\rangle > \langle\boldsymbol{c}, \boldsymbol{\mu}\rangle
$$

Therefore:
$$
\langle\boldsymbol{c}, \boldsymbol{\mu}\rangle - \gamma\langle\boldsymbol{v}, \boldsymbol{\mu}\rangle < 0
$$

Substituting into the difference formula:
$$
\|\gamma \boldsymbol{v} - \boldsymbol{\mu}\|^2 < \|\boldsymbol{c} - \boldsymbol{\mu}\|^2
$$

That is:
$$
\|\boldsymbol{g}_{(t,s_i)}' - \nabla F(\boldsymbol{w}_t)\|^2 < \|\boldsymbol{g}_{(t,s_i)} - \nabla F(\boldsymbol{w}_t)\|^2
$$

Theorem proved.
\end{proof}

\section{Supplementary Algorithms}
\label{sec:extend-algo}
\SetKwInput{KwData}{Require}
\begin{algorithm}[htbp]
\setlength{\abovedisplayskip}{2pt}
\setlength{\belowdisplayskip}{2pt}
\setlength{\abovedisplayshortskip}{0pt}
\setlength{\belowdisplayshortskip}{0pt}
\caption{FedProx-ECGR}
\label{alg:FedProx-ECGR}
\KwData{Total global rounds $T$, local dataset $\mathcal{D}_{i}$ ($\boldsymbol{x}_{i}\in\mathcal{D}_{i}$), local iterations $\tau_i$, initialized weight $\boldsymbol{w}_{0}$, initialized order $ s_i $ at client $i$, learning rate $\eta > 0$, proximal coefficient $\mu > 0$}
  \For{each round $t=0,\dots,T-1$}
  {
    Parameter server sends the global model $\boldsymbol{w}_{t}$ to all participating clients\;
    \For{each client $i=1,...,N$}
    {
      \For{each local iteration $\lambda=0,1,\dots,\tau_i$}
      {
        Initialize the local model $\boldsymbol{w}_{(t,i)}^{\lambda} \leftarrow \boldsymbol{w}_{t}$ \;
        Local update with proximal term:
        $ \boldsymbol{w}_{(t,i)}^{\lambda+1} = \boldsymbol{w}_{(t,i)}^{\lambda} - \eta \left( \nabla F_i(\boldsymbol{w}_{(t,i)}^{\lambda}; \boldsymbol{x}_{s_i}) + \mu (\boldsymbol{w}_{(t,i)}^{\lambda} - \boldsymbol{w}_{t}) \right)$ \;
      }
      Store the local gradient set $ \left\{ \eta \left( \nabla F_i(\boldsymbol{w}_{(t,i)}^{\lambda}; \boldsymbol{x}_{s_i}) + \mu (\boldsymbol{w}_{(t,i)}^{\lambda} - \boldsymbol{w}_{t}) \right) \right\}_{\lambda=1}^{\tau_i} $ \;
      \textbf{ECGR}:
          \[
          \begin{split}
          &\pi_i \leftarrow \arg\min_{\pi_i \subset s_i} \big\| \boldsymbol{g}_{(t,\pi_i)} \big\|
          \quad\#\ \text{Magnitude Ranking}
          \\
          &\pi_i^{\prime} = s_i\setminus \pi_i,  \quad \beta \boldsymbol{g}_{(t,\pi_i^{\prime})}
          \quad\#\ \text{Attenuated Extraction}
          \\
          &\boldsymbol{g}_{(t,s_i)}^{\prime}=\gamma_i(\boldsymbol{g}_{(t,\pi_i)} + \beta \boldsymbol{g}_{(t,\pi_i^{\prime})})
          \quad\#\ \text{Re-aggregation}
          \end{split}
          \]
    }
    Parameter server receives $\boldsymbol{g}_{(t,s_i)}^{\prime}$ from all clients\;
    Global aggregation $\boldsymbol{G}_{t}^{\prime} = \sum_{i=1}^N p_i \boldsymbol{g}_{(t,s_i)}^{\prime}$ \;
    Global update $\boldsymbol{w}_{t+1} = \boldsymbol{w}_{t} - \boldsymbol{G}_{t}^{\prime}$\;
  }
  \Return $\boldsymbol{w}_{T}$\;
\end{algorithm}

\begin{algorithm}[htbp]
\setlength{\abovedisplayskip}{2pt}
\setlength{\belowdisplayskip}{2pt}
\setlength{\abovedisplayshortskip}{0pt}
\setlength{\belowdisplayshortskip}{0pt}
\caption{FedNova-ECGR}
\label{alg:FedNova-ECGR}
\KwData{Total global rounds $T$, local dataset $\mathcal{D}_{i}$ ($\boldsymbol{x}_{i}\in\mathcal{D}_{i}$), local iterations $\tau_i$, initialized weight $\boldsymbol{w}_{0}$, initialized order $ s_i $ at client $i$, learning rate $\eta > 0$}
  \For{each round $t=0,\dots,T-1$}
  {
    Parameter server sends the global model $\boldsymbol{w}_{t}$ to all participating clients\;
    \For{each client $i=1,...,N$}
    {
      Initialize the local model $\boldsymbol{w}_{(t,i)}^{0} \leftarrow \boldsymbol{w}_{t}$ \;
      \For{each local iteration $\lambda=0,1,\dots,\tau_i-1$}
      {
        Local update:
        $ \boldsymbol{w}_{(t,i)}^{\lambda+1} = \boldsymbol{w}_{(t,i)}^{\lambda} - \eta \nabla F_i(\boldsymbol{w}_{(t,i)}^{\lambda}; \boldsymbol{x}_{s_i}) $\;
      }
      Store the local gradient set $ \left\{ \eta \nabla F_i(\boldsymbol{w}_{(t,i)}^\lambda; \boldsymbol{x}_{s_i}) \right\}_{\lambda=1}^{\tau_i} $ \;

      \textbf{ECGR}:
      \[
      \begin{split}
      &\pi_i \leftarrow \arg\min_{\pi_i \subset s_i} \big\| \boldsymbol{g}_{(t,\pi_i)} \big\|
      \quad\#\ \text{Magnitude Ranking}
      \\
      &\pi_i^{\prime} = s_i\setminus \pi_i,  \quad \beta \boldsymbol{g}_{(t,\pi_i^{\prime})}
      \quad\#\ \text{Attenuated Extraction}
      \\
      &\boldsymbol{g}_{(t,s_i)}^{\prime}=\gamma_i(\boldsymbol{g}_{(t,\pi_i)} + \beta \boldsymbol{g}_{(t,\pi_i^{\prime})})
      \quad\#\ \text{Re-aggregation}
      \end{split}
      \]
        Normalize by local steps: $\boldsymbol{g}_{(t,s_i)}^{\prime} \leftarrow \frac{\boldsymbol{g}_{(t,s_i)}^{\prime}}{\tau_i}$\;
    }
    Parameter server receives $\boldsymbol{g}_{(t,s_i)}^{\prime}$ and $\tau_i$ from all clients\;
    Compute effective step size:
    $\tau_{\mathrm{eff}} = \sum_{i=1}^N p_i \tau_i$\;
    Aggregate normalized gradients:
    $\boldsymbol{G}_{t}^{\prime} = \tau_{\mathrm{eff}} \sum_{i=1}^N p_i \boldsymbol{g}_{(t,s_i)}^{\prime}$\;
    Global update $\boldsymbol{w}_{t+1} = \boldsymbol{w}_{t} - \boldsymbol{G}_{t}^{\prime}$\;
  }
  \Return $\boldsymbol{w}_{T}$\;
\end{algorithm}

\begin{algorithm}[htbp]
\setlength{\abovedisplayskip}{2pt}
\setlength{\belowdisplayskip}{2pt}
\setlength{\abovedisplayshortskip}{0pt}
\setlength{\belowdisplayshortskip}{0pt}
\caption{Scaffold-ECGR}
\label{alg:Scaffold-ECGR}
\KwData{Total global rounds $T$, local dataset $\mathcal{D}_{i}$ ($\boldsymbol{x}_{i}\in\mathcal{D}_{i}$), local iterations $\tau_i$, initialized weight $\boldsymbol{w}_{0}$, global control variate $c$, local control variates $c_i$, learning rate $\eta > 0$}
  \For{each round $t=0,\dots,T-1$}
  {
    Parameter server sends $(\boldsymbol{w}_{t}, c)$ to all participating clients\;
    \For{each client $i=1,...,N$}
    {
      Initialize the local model $\boldsymbol{w}_{(t,i)}^{0} \leftarrow \boldsymbol{w}_{t}$\;
      \For{each local iteration $\lambda=0,1,\dots,\tau_i-1$}
      {
        Local update with control correction:
        \[
        \boldsymbol{w}_{(t,i)}^{\lambda+1} = 
        \boldsymbol{w}_{(t,i)}^{\lambda} - 
        \eta \left( \nabla F_i(\boldsymbol{w}_{(t,i)}^{\lambda}; \boldsymbol{x}_{s_i}) - c_i + c \right)
        \]
      }
      Store the corrected local gradient set 
      $ \left\{ \eta \left( \nabla F_i(\boldsymbol{w}_{(t,i)}^{\lambda}; \boldsymbol{x}_{s_i}) - c_i + c \right) \right\}_{\lambda=1}^{\tau_i} $ \;

      \textbf{ECGR}:
      \[
      \begin{split}
      &\pi_i \leftarrow \arg\min_{\pi_i \subset s_i} \big\| \boldsymbol{g}_{(t,\pi_i)} \big\|
      \quad\#\ \text{Magnitude Ranking}
      \\
      &\pi_i^{\prime} = s_i\setminus \pi_i,  \quad \beta \boldsymbol{g}_{(t,\pi_i^{\prime})}
      \quad\#\ \text{Attenuated Extraction}
      \\
      &\boldsymbol{g}_{(t,s_i)}^{\prime}=\gamma_i(\boldsymbol{g}_{(t,\pi_i)} + \beta \boldsymbol{g}_{(t,\pi_i^{\prime})})
      \quad\#\ \text{Re-aggregation}
      \end{split}
      \]

      Update the local control variate:
      \[
      c_i^{\prime} = c_i - c + \frac{1}{\tau_i \eta}(\boldsymbol{w}_{t} - \boldsymbol{w}_{(t,i)}^{\tau_i})
      \]
    }
    Parameter server receives $\boldsymbol{g}_{(t,s_i)}^{\prime}$ and $c_i^{\prime}$ from all clients\;
    Global aggregation:
    $\boldsymbol{G}_{t}^{\prime} = \sum_{i=1}^{N} p_i \boldsymbol{g}_{(t,s_i)}^{\prime}$\;
    Global model update:
    $\boldsymbol{w}_{t+1} = \boldsymbol{w}_{t} - \boldsymbol{G}_{t}^{\prime}$\;
    Update global control variate:
    $c \leftarrow \sum_{i=1}^{N} p_i c_i^{\prime}$\;
  }
  \Return $\boldsymbol{w}_{T}$\;
\end{algorithm}

The ECGR-extended variants of FedProx, FedNova, and Scaffold are provided in 
\cref{alg:FedProx-ECGR}, \cref{alg:FedNova-ECGR}, and \cref{alg:Scaffold-ECGR}, respectively. 
Consistent with \cref{alg:FedAvg-ECGR}, ECGR does not alter the fundamental training 
procedure of the original baselines; rather, it introduces an additional 
gradient-selection stage. This design likewise preserves the \emph{communication 
efficiency} and \emph{structural compatibility} properties highlighted previously in 
\cref{alg:FedAvg-ECGR}.

For other advanced FL algorithms not covered in this work, as well as future developments in federated optimization, the ECGR extension can be readily constructed by following the design principles illustrated in this section.

\section{Additional Results}
\label{app:additional_results}
\subsection{Benchmark Image Classification}

In this section, we provide additional experimental results for benchmarking on standard image classification datasets. In particular, we present more comprehensive ablation studies on CIFAR-10, complementing the analyses reported in the main paper. Moreover, we further include baseline ablations on MNIST, Fashion-MNIST, and CIFAR-100, which were not discussed in the main text.

\subsubsection{CIFAR-10}

\begin{figure}[htbp]
    \centering
    \begin{subfigure}{0.32\textwidth}
        \centering
        \includegraphics[width=\textwidth]{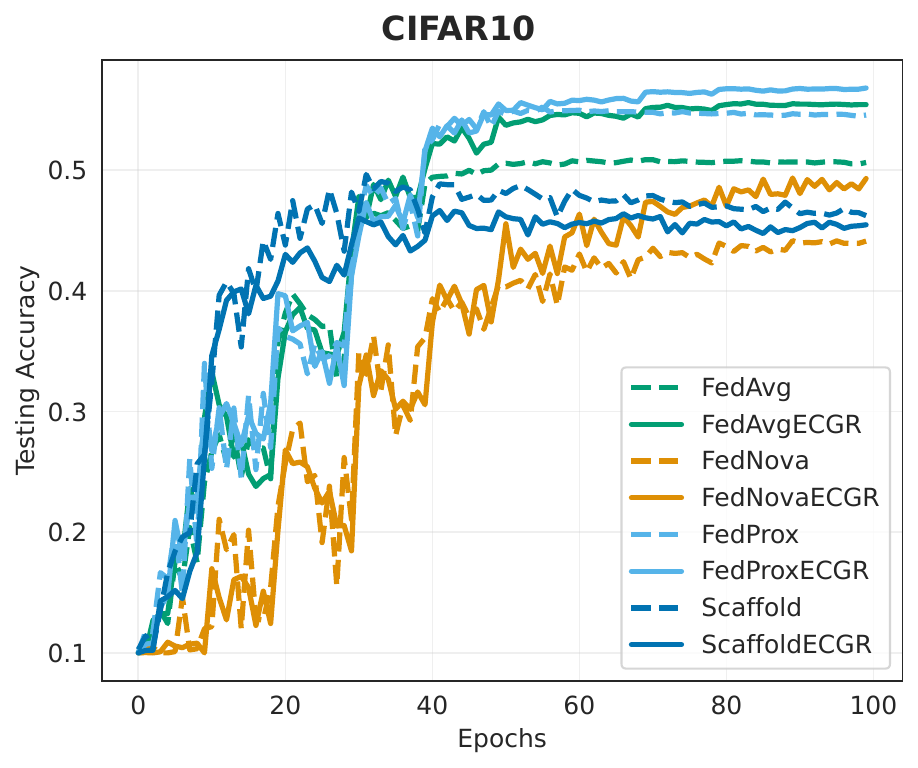}
        \caption{$\eta=0.01$}
    \end{subfigure}
    \hfill
    \begin{subfigure}{0.32\textwidth}
        \centering
        \includegraphics[width=\textwidth]{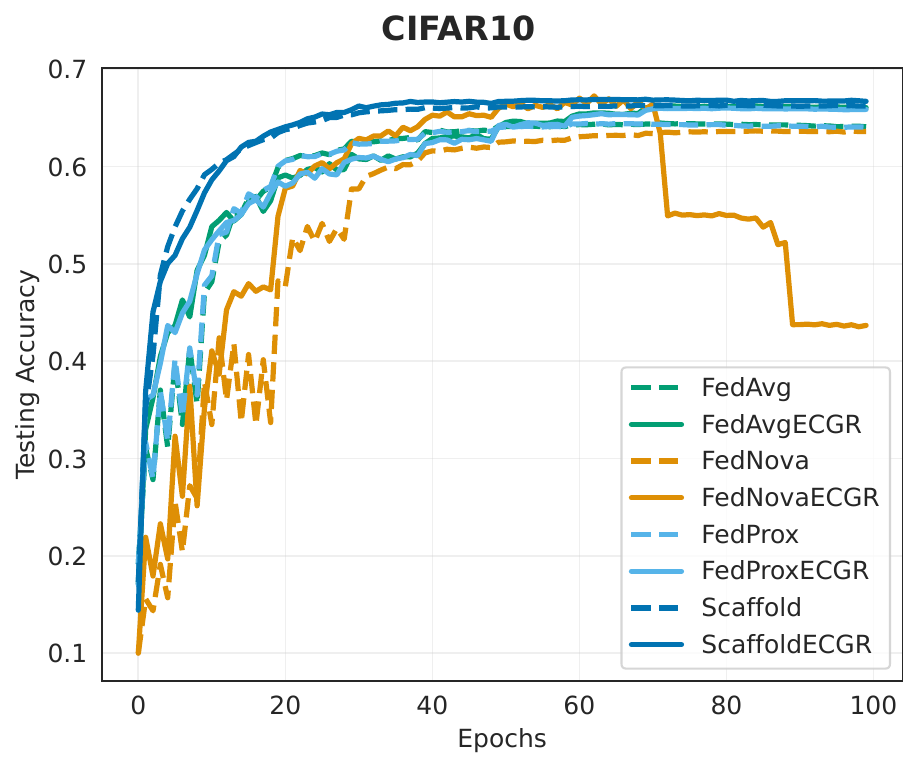}
        \caption{$\alpha=0.1$}
    \end{subfigure}
    \hfill
    \begin{subfigure}{0.32\textwidth}
        \centering
        \includegraphics[width=\textwidth]{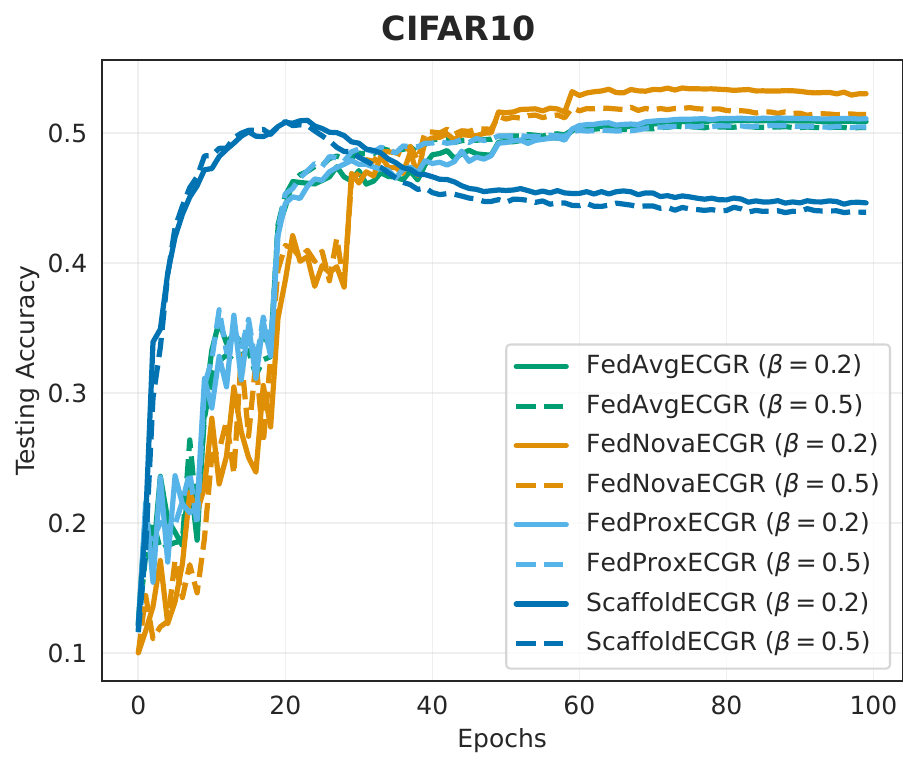}
        \caption{$\beta=0.5$ vs. $\beta=0.2$}
    \end{subfigure}

    \caption{Ablation studies on CIFAR-10 with respect to learning rate $\eta$, data heterogeneity level $\alpha$, and the ECGR damping coefficient $\beta$. All curves report the mean test accuracy over five independent runs with random seeds 0, 1, 42, 999, and 2025.}
    \label{fig:cifar10_ablation_extra}
\end{figure}

As shown in \cref{fig:cifar10_ablation_extra}~(a), and in comparison with 
\cref{fig:cifar10_ablation}~(a), the upper bound of the average accuracy becomes 
higher; however, the performance of Scaffold-ECGR further deteriorates. In contrast, 
the ECGR variants of FedAvg and FedProx remain effective, and the catastrophic 
accuracy drop observed in FedNova disappears. These results indicate that, for 
different baselines—particularly those that manipulate local gradients directly—
careful tuning of the learning rate is essential.

Compared with \cref{fig:cifar10_ablation}~(b), \cref{fig:cifar10_ablation_extra}~(b) shows that FedNova again suffers a catastrophic drop in accuracy. This indicates that FedNova-ECGR requires more sensitive and adaptive hyperparameter tuning under different Dirichlet partitions. Nevertheless, our ECGR strategy still provides a modest performance gain, further suggesting that its benefits become more pronounced as the degree of data heterogeneity increases.

\cref{fig:cifar10_ablation_extra}~(c) shows that as the damping coefficient $\beta$ increases, 
the performance gain provided by ECGR gradually weakens and eventually degenerates to the baseline. 
However, \cref{fig:cifar10_ablation} demonstrates that an excessively small $\beta$ leads to 
accuracy oscillations and slower improvement in the early training stage. 
Therefore, selecting an appropriate $\beta$ is essential to balance the gain of ECGR and 
the instability caused by discarding too much gradient information.

\subsubsection{MNIST}

\begin{figure}[htbp]
    \centering
    \begin{subfigure}{0.32\textwidth}
        \centering
        \includegraphics[width=\textwidth]{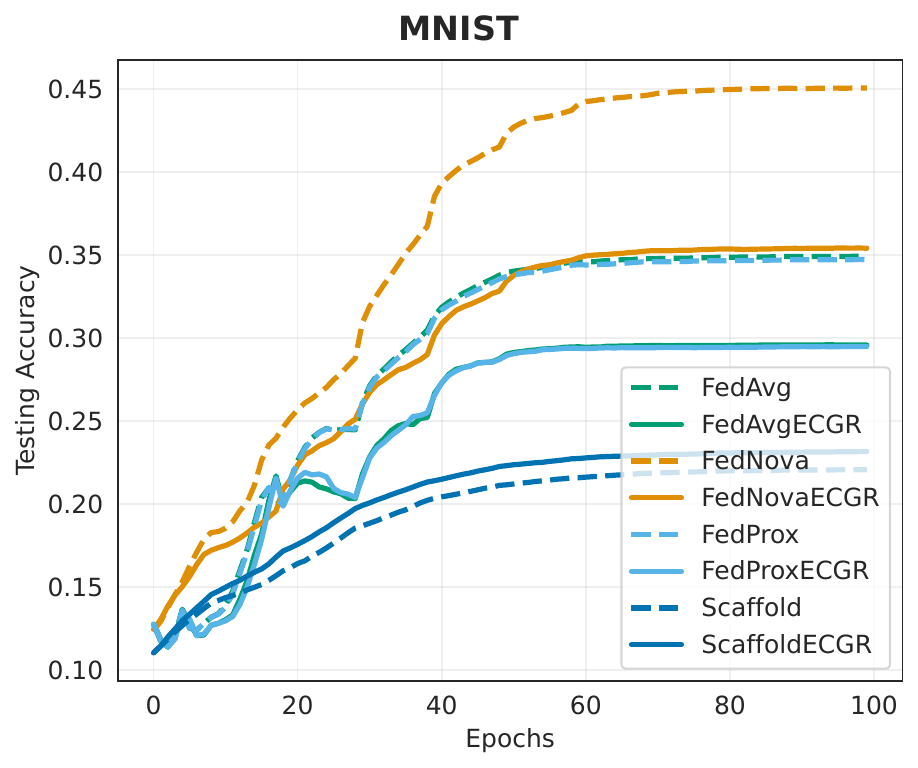}
        \caption{$\eta=0.0001$}
    \end{subfigure}
    \hfill
    \begin{subfigure}{0.32\textwidth}
        \centering
        \includegraphics[width=\textwidth]{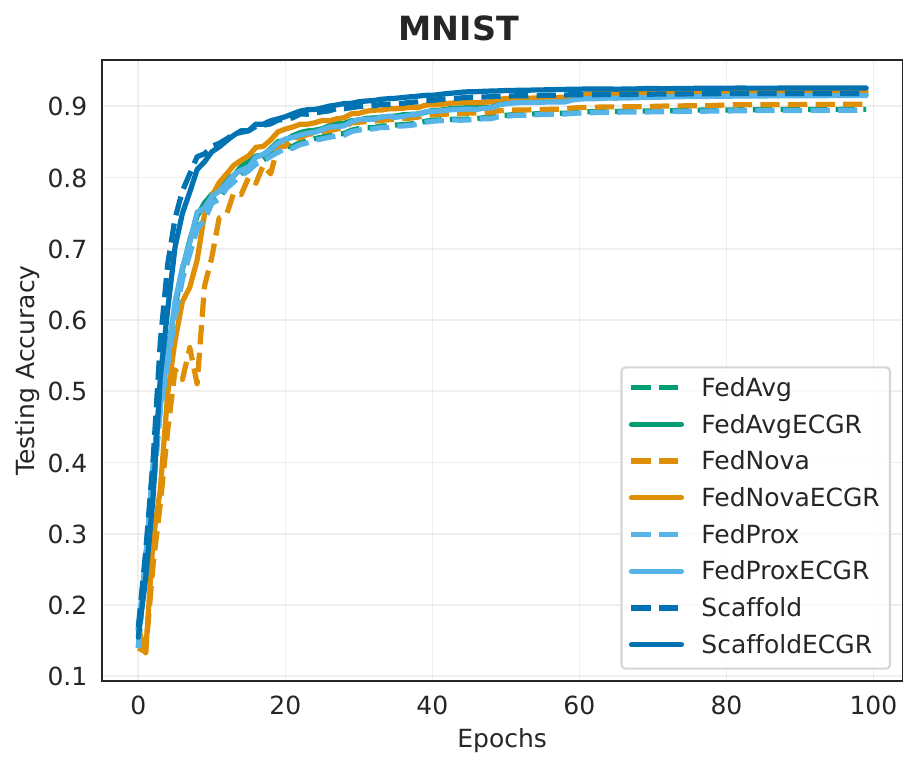}
        \caption{$\alpha=0.1$}
    \end{subfigure}
    \hfill
    \begin{subfigure}{0.32\textwidth}
        \centering
        \includegraphics[width=\textwidth]{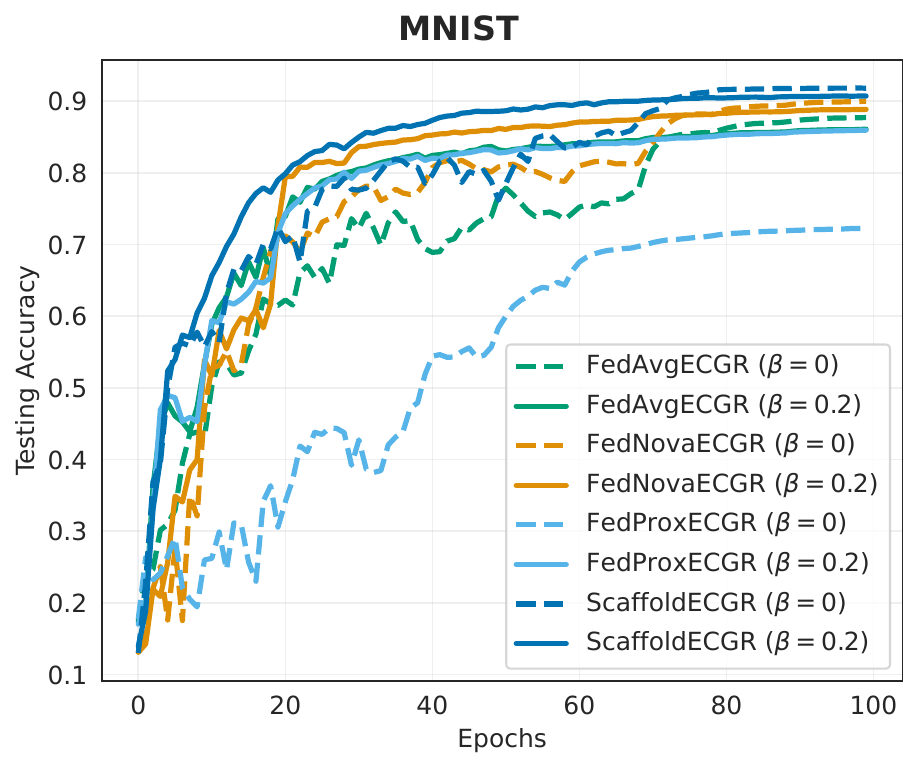}
        \caption{$\beta=0$ vs. $\beta=0.2$}
    \end{subfigure}

    \caption{Ablation studies on MNIST with respect to learning rate $\eta$, data heterogeneity level $\alpha$, and the ECGR damping coefficient $\beta$. All curves report the mean test accuracy over five independent runs with random seeds 0, 1, 42, 999, and 2025.}
    \label{fig:mnist_ablation}
\end{figure}
Because the MNIST dataset is overly simple and the LeNet model is relatively small, the training dynamics become highly sensitive to the choice of learning rate. As a result, when the learning rate is set too low, the baselines fail to converge, as shown in \cref{fig:mnist_ablation}~(a). Under such circumstances, the ECGR strategy cannot provide valid improvements.

As shown in \cref{fig:mnist_ablation}~(b), the results are consistent with our findings on CIFAR-10 dataset, ECGR exhibits better performance under highly non-IID data settings.

As shown in \cref{fig:mnist_ablation}~(c), the results on MNIST follow the same trend observed on CIFAR-10: discarding exploratory gradients (i.e., $\beta = 0$) leads to noticeable accuracy oscillations and degradation during the early stage of training. 
However, due to the simplicity of the MNIST dataset and the small parameter size of the LeNet model—which together reduce the optimization difficulty and lessen the negative effects of losing gradient information—a setting of $\beta = 0$ unexpectedly yields improved final performance for most baselines (except FedProx-ECGR).

\subsubsection{Fashion-MNIST}

\begin{figure}[htbp]
    \centering
    \begin{subfigure}{0.32\textwidth}
        \centering
        \includegraphics[width=\textwidth]{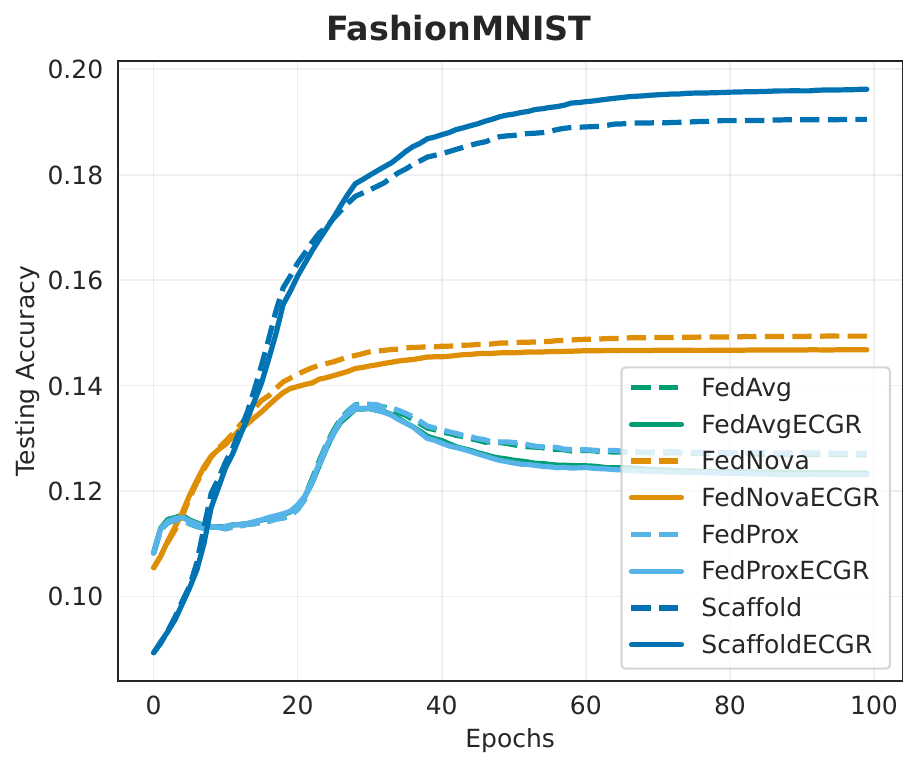}
        \caption{$\eta=0.0001$}
    \end{subfigure}
    \hfill
    \begin{subfigure}{0.32\textwidth}
        \centering
        \includegraphics[width=\textwidth]{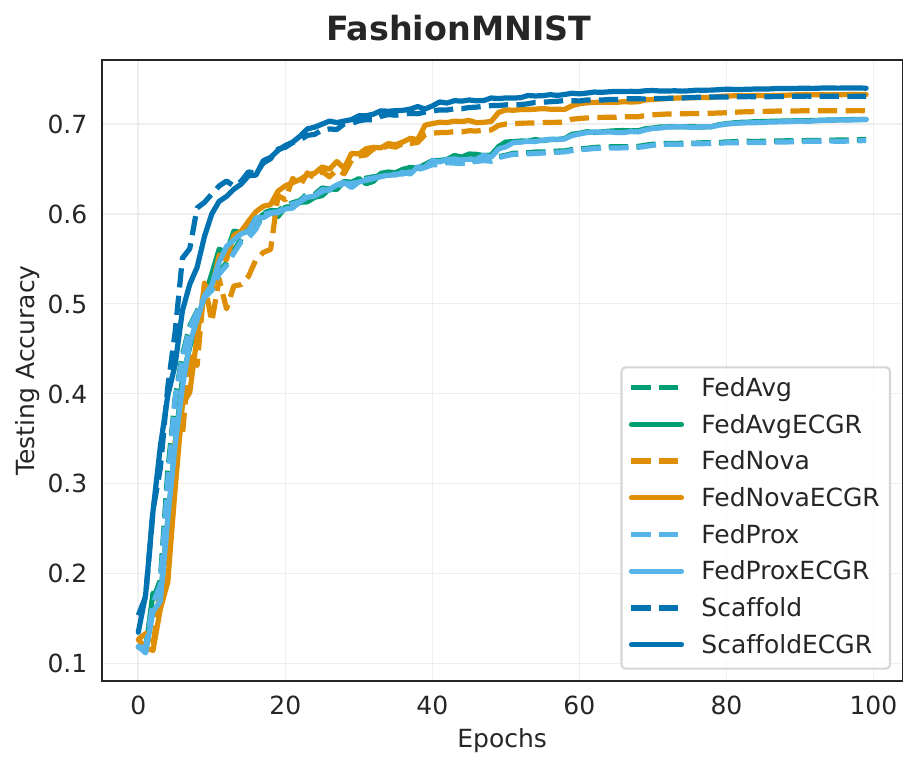}
        \caption{$\alpha=0.1$}
    \end{subfigure}
    \hfill
    \begin{subfigure}{0.32\textwidth}
        \centering
        \includegraphics[width=\textwidth]{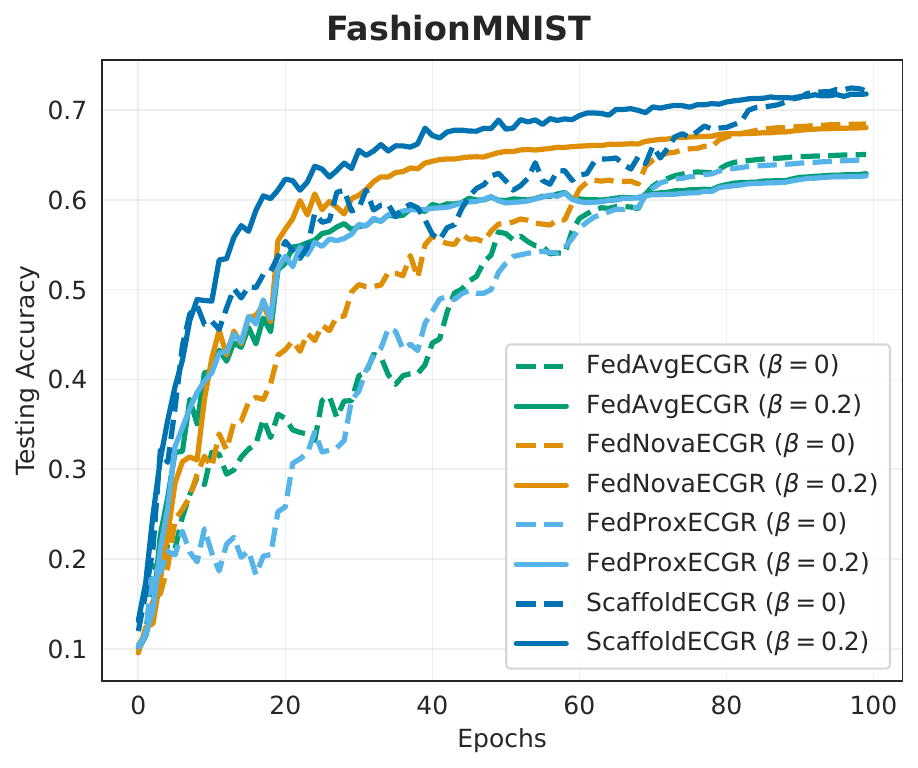}
        \caption{$\beta=0$ vs. $\beta=0.2$}
    \end{subfigure}

    \caption{Ablation studies on Fashion-MNIST with respect to learning rate $\eta$, data heterogeneity level $\alpha$, and the ECGR damping coefficient $\beta$. All curves report the mean test accuracy over five independent runs with random seeds 0, 1, 42, 999, and 2025.}
    \label{fig:fmnist_ablation}
\end{figure}

The analytical findings on Fashion-MNIST exhibit the same overall trends as those observed on MNIST.

\subsubsection{CIFAR-100}

\begin{figure}[htbp]
    \centering
    \begin{subfigure}{0.32\textwidth}
        \centering
        \includegraphics[width=\textwidth]{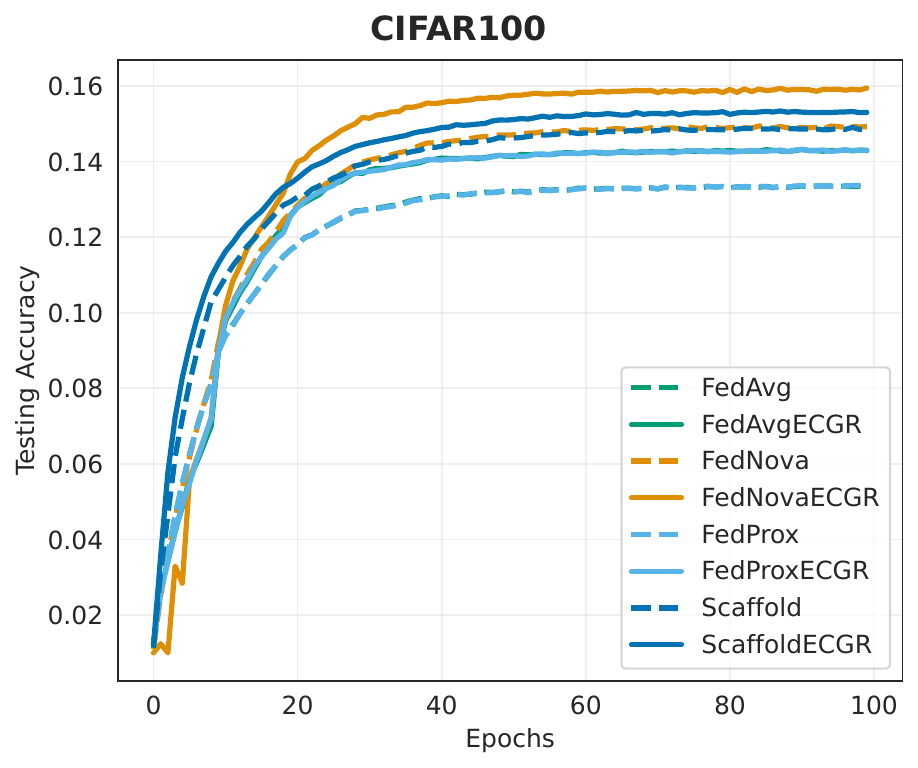}
        \caption{$\eta=0.0001$}
    \end{subfigure}
    \hfill
    \begin{subfigure}{0.32\textwidth}
        \centering
        \includegraphics[width=\textwidth]{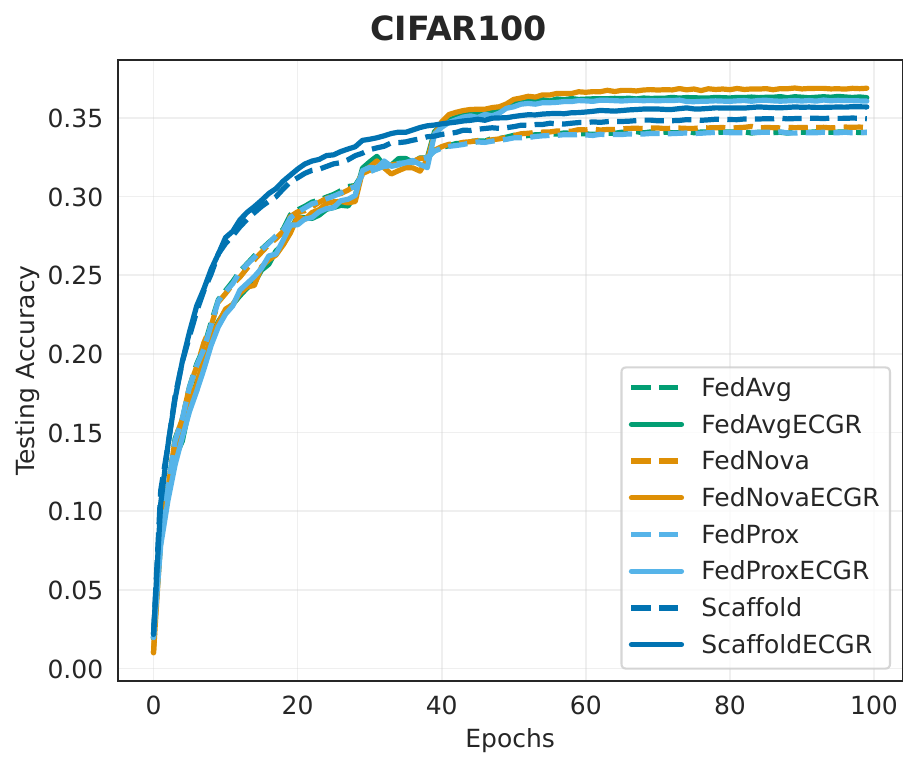}
        \caption{$\alpha=0.1$}
    \end{subfigure}
    \hfill
    \begin{subfigure}{0.32\textwidth}
        \centering
        \includegraphics[width=\textwidth]{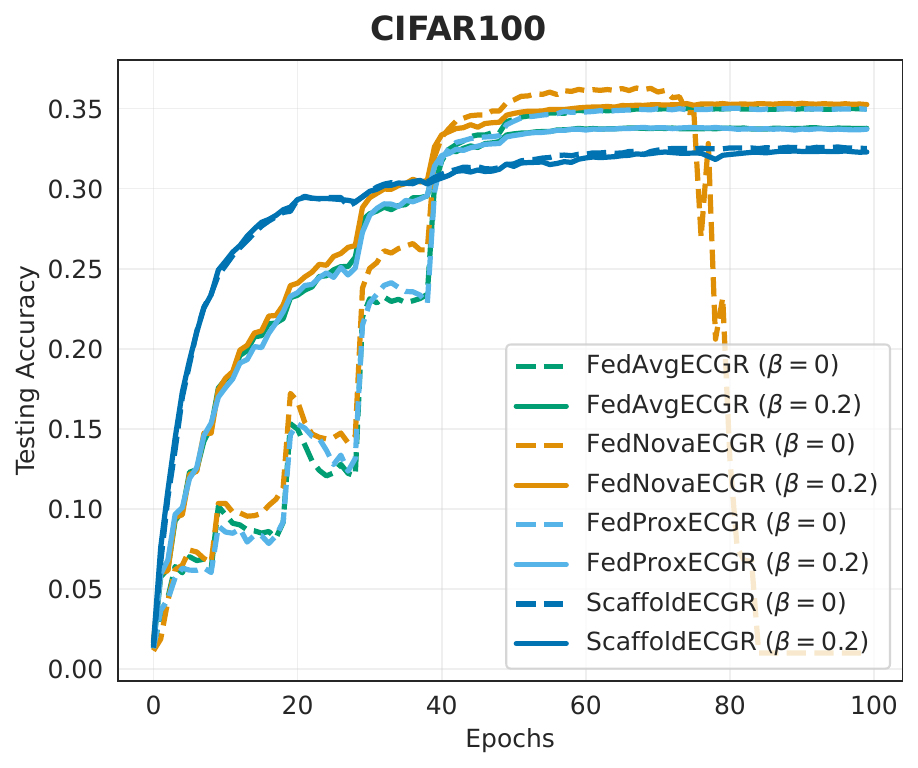}
        \caption{$\beta=0$ vs. $\beta=0.2$}
    \end{subfigure}

    \caption{Ablation studies on CIFAR-100 with respect to learning rate $\eta$, data heterogeneity level $\alpha$, and the ECGR damping coefficient $\beta$. All curves report the mean test accuracy over five independent runs with random seeds 0, 1, 42, 999, and 2025.}
    \label{fig:cifar100_ablation}
\end{figure}

The conclusions drawn from CIFAR-10 and CIFAR-100 are largely consistent, 
except for those related to the learning rate~$\eta$. Similar to MNIST and Fashion-MNIST, 
the upper bound of the average test accuracy on CIFAR-100 decreases substantially. 
However, unlike these simpler datasets, ECGR still provides a noticeable performance gain. 
This can be attributed to the higher complexity and richer semantic diversity of CIFAR-100, 
as well as the larger capacity of the CNN models employed, which make the distinction between exploratory and convergent gradient phases more pronounced and allow ECGR to better exploit this structure.

\subsection{Additional Visualization of Gradient Selection in ECGR}
In this section, we extend the gradient–selection visualizations presented in
\cref{sec:visualization} for CIFAR-10. We first provide the 3D views of the
selected gradients under the IID setting, followed by 3D visualizations
obtained with different random seeds. We then further present the 3D views on
additional datasets under the seed--42 setting.

\subsubsection{CIFAR-10}

\begin{figure}[htbp]
    \centering
    \begin{subfigure}{0.32\textwidth}
        \centering
        \includegraphics[width=\textwidth]{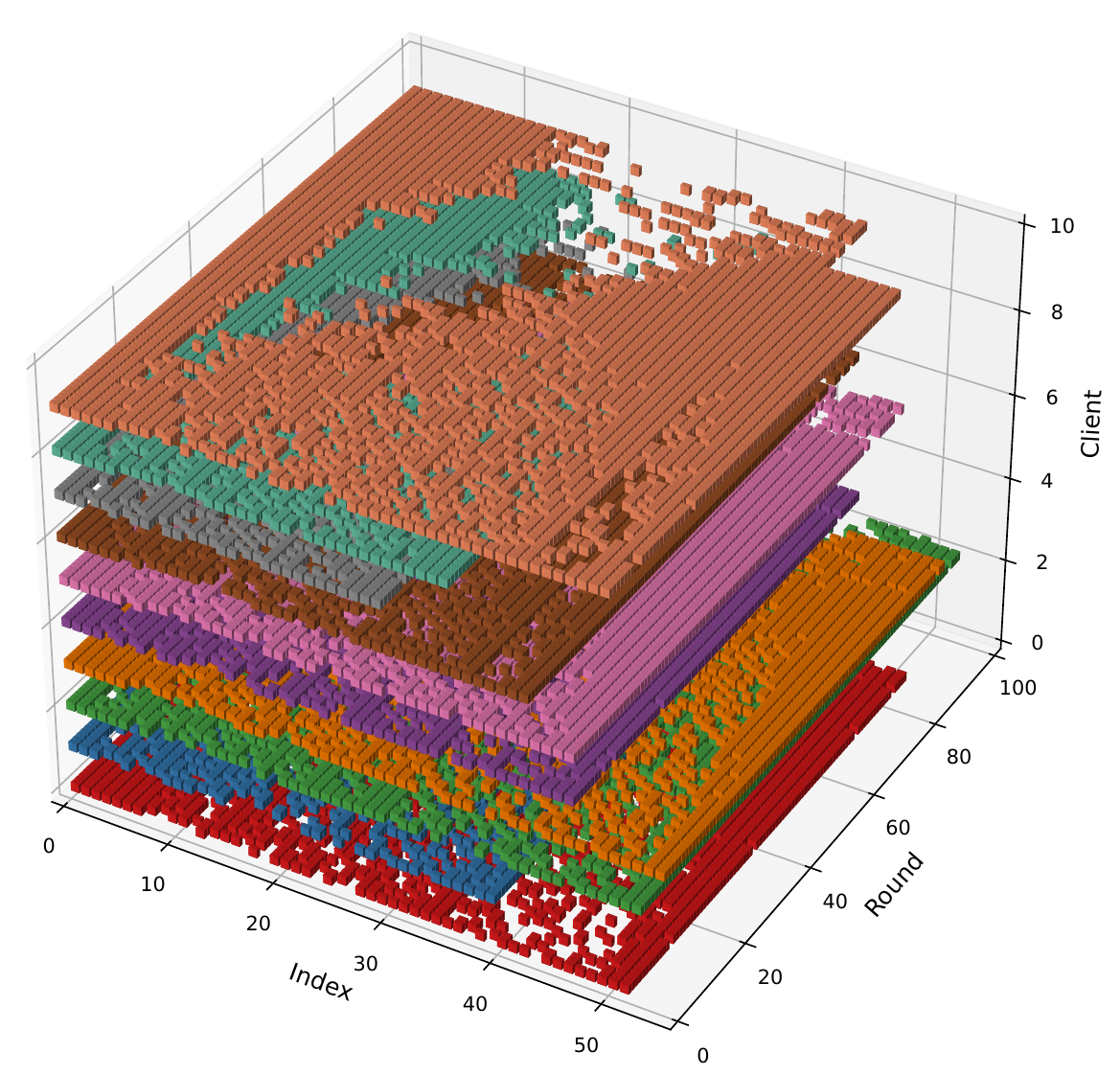}
        \caption{$\alpha=1$}
    \end{subfigure}
    \hfill
    \begin{subfigure}{0.32\textwidth}
        \centering
        \includegraphics[width=\textwidth]{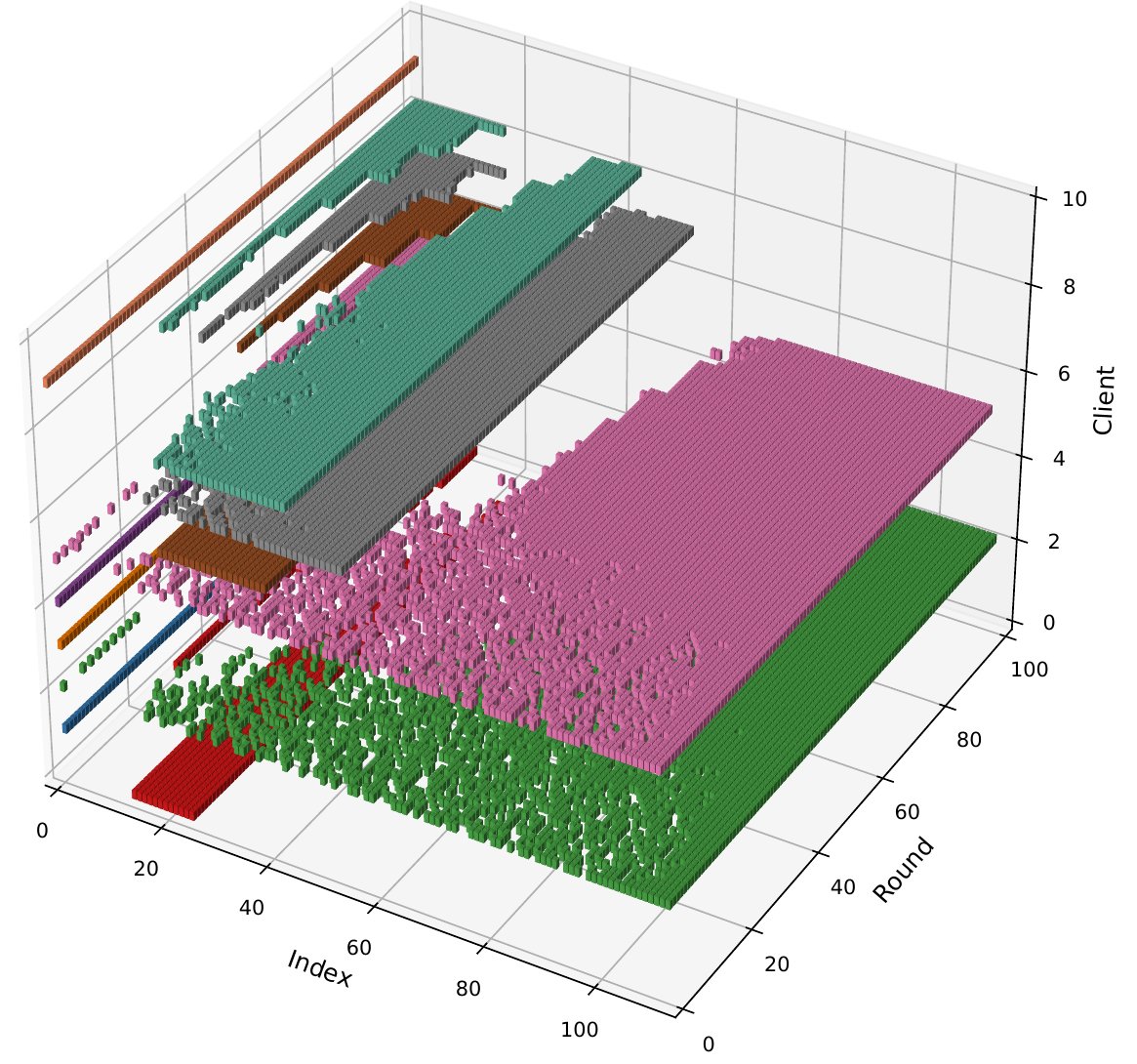}
        \caption{Seed 0}
    \end{subfigure}
    \hfill
    \begin{subfigure}{0.32\textwidth}
        \centering
        \includegraphics[width=\textwidth]{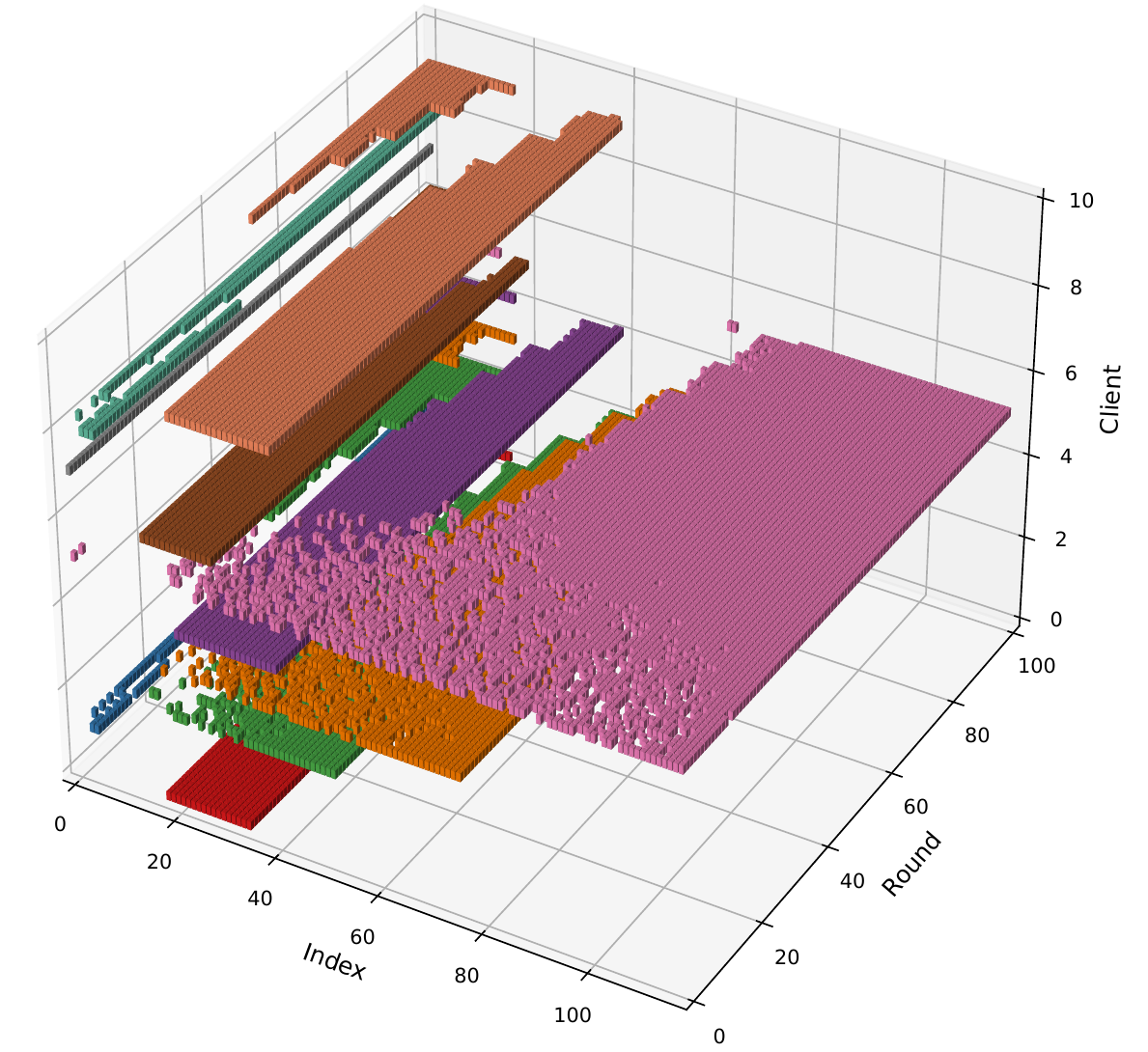}
        \caption{Seed 2025}
    \end{subfigure}

    \caption{Visualization of ECGR’s gradient selection under both IID (\(\alpha=1\), seed = 42) and non-IID (\(\alpha=0.01\), seeds = 0 and 2025) settings on CIFAR-10.}

    \label{fig:grad_sel_vision_extra}
\end{figure}
As shown in \cref{fig:grad_sel_vision_extra}~(a), in an almost IID setting (i.e., $\alpha = 1$), the gradient selections made by ECGR tend to resemble random choices. This is because, under IID data distribution, the gradients computed on each client are already close to the optimal gradient. Consequently, the discrepancy between exploratory and convergent gradients becomes small, leading to weaker distinguishability and more uniformly mixed selections.

The visualizations in \cref{fig:grad_sel_vision_extra}~(b) and (c) indicate that, under non-IID settings, the variation in client data distributions induced by different random seeds has only a minor impact on the gradient-selection behavior of ECGR. Across all seeds, ECGR consistently prefers gradients from later local iterations as the convergence-oriented gradients, which aligns with the classical convergence behavior of SGD.

\subsubsection{MNIST, Fashion-MNIST and CIFAR-100}

\begin{figure}[htbp]
    \centering
    \begin{subfigure}{0.32\textwidth}
        \centering
        \includegraphics[width=\textwidth]{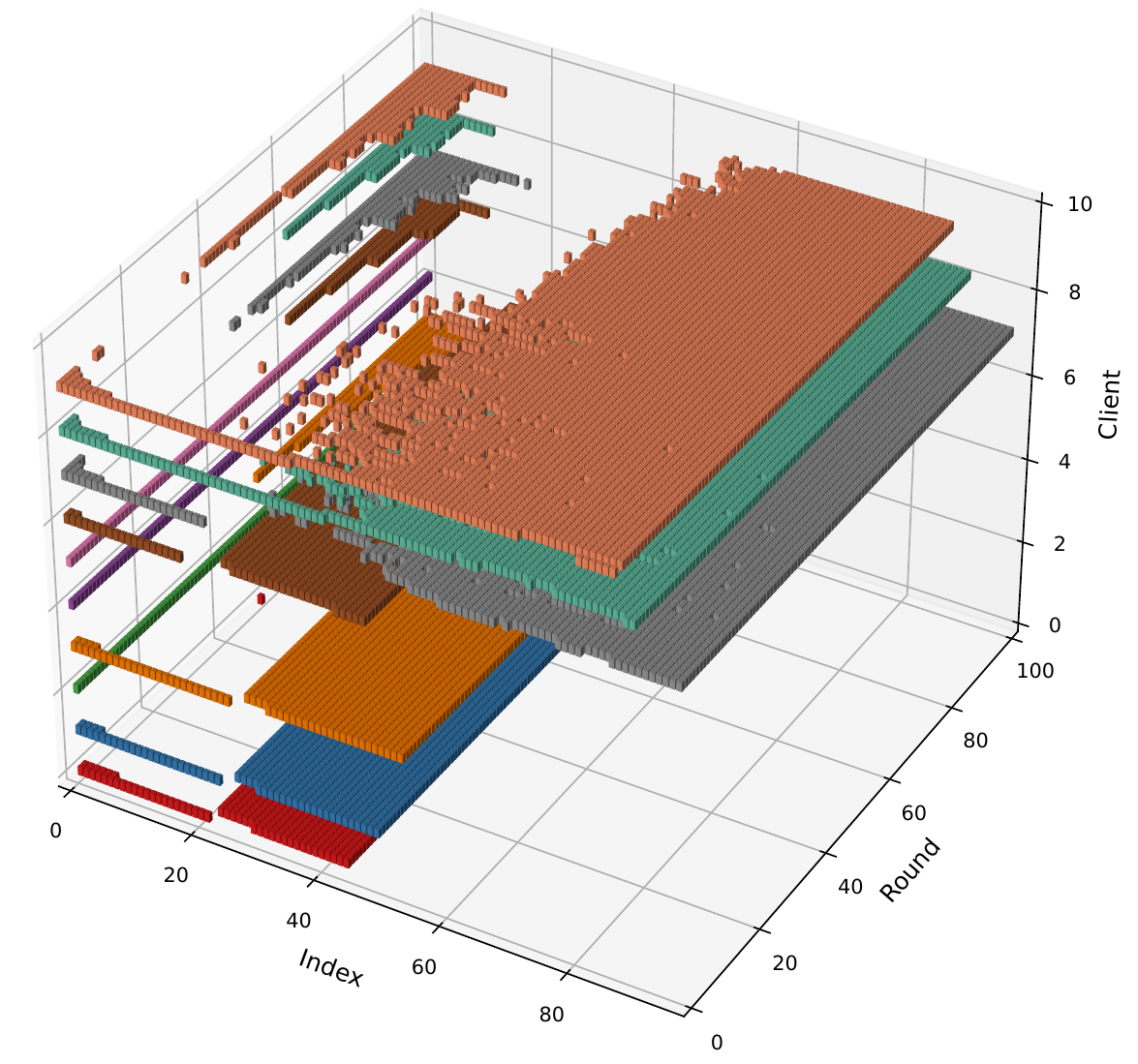}
        \caption{MNIST}
    \end{subfigure}
    \hfill
    \begin{subfigure}{0.32\textwidth}
        \centering
        \includegraphics[width=\textwidth]{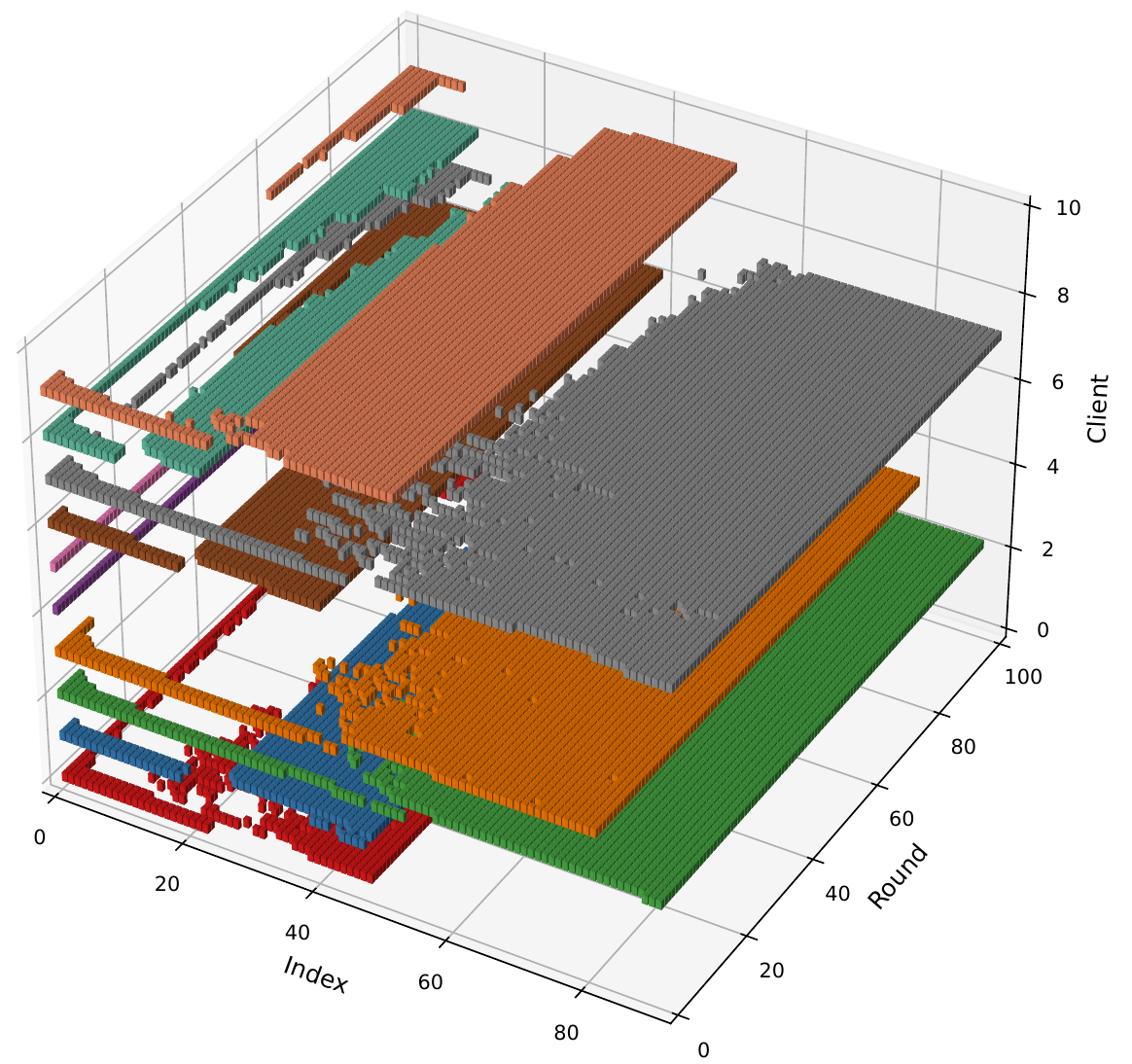}
        \caption{Fashion-MNIST}
    \end{subfigure}
    \hfill
    \begin{subfigure}{0.32\textwidth}
        \centering
        \includegraphics[width=\textwidth]{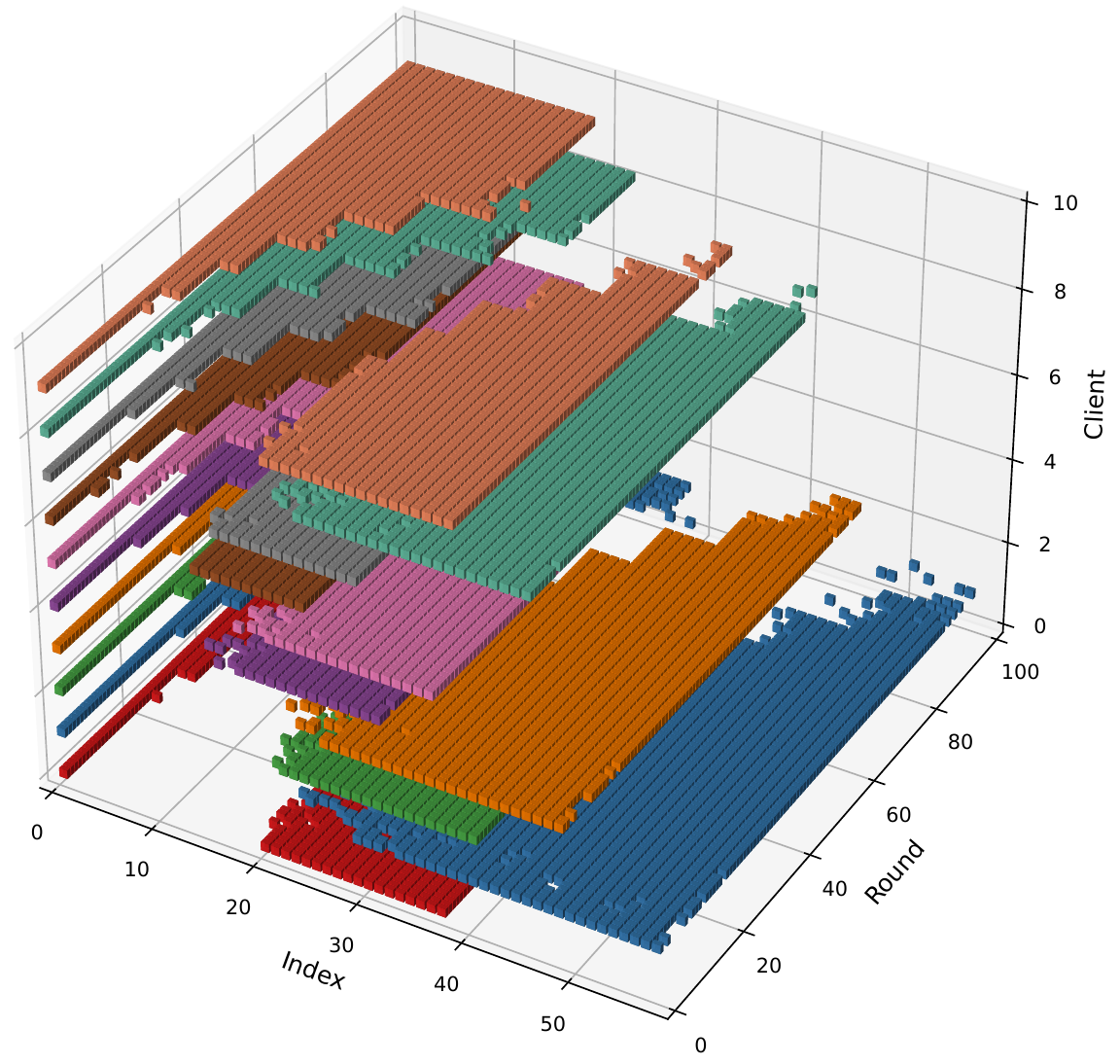}
        \caption{CIFAR-100}
    \end{subfigure}

    \caption{Visualization of ECGR's per-round gradient selection on three datasets---MNIST, Fashion-MNIST, and CIFAR-100. All visualizations are generated under the same experimental setting with Dirichlet heterogeneity parameter $\alpha = 0.01$ and random seed $42$.}
    \label{fig:grad_sel_vision_datasets}
\end{figure}

As shown in \cref{fig:grad_sel_vision_datasets}, the MNIST and Fashion-MNIST datasets exhibit gradient-selection behaviors consistent with those observed on CIFAR-10. However, for CIFAR-100, the convergence gradients selected by ECGR tend to correspond to later local iterations during the early stage of training, whereas in the later stage—when the global model approaches convergence—the selected convergence gradients shift toward earlier local iterations. A mild version of this phenomenon also appears in the other three datasets (MNIST, Fashion-MNIST, and CIFAR-10). 

This behavior can instead be explained by the observation that, in the later stages of training, the global model gradually approaches a convergent regime. As a consequence, the discriminative gap between exploratory and convergent gradients diminishes, causing the selected convergent gradients to shift toward earlier local iterations. This shift becomes more evident on CIFAR-100 due to its higher task complexity, which accelerates the onset of this near-convergence behavior.

\end{appendix}

\end{document}